\documentclass[twoside,11pt]{article}

\usepackage{jmlr2e}
\usepackage{color}
\usepackage{bm}
\usepackage{amsmath}
\usepackage{algorithm}
\usepackage{algorithmic}
\usepackage{hyphenat}

\newtheorem{thm}{Theorem}

\newtheorem{assumption}{Assumption}

\ShortHeadings{A Bayesian Approach for Online Classifier Ensemble}{Bai, Lam and Sclaroff}
\firstpageno{1}

\begin{document}

\title{A Bayesian Approach for Online Classifier Ensemble}

\author{\name Qinxun Bai \email qinxun@cs.bu.edu \\
       \addr Department of Computer Science\\
       Boston University\\
       Boston, MA 02215, USA
       \AND
       \name Henry Lam \email khlam@umich.edu \\
       \addr Department of Industrial \& Operations Engineering\\
       University of Michigan\\
       Ann Arbor, MI 48109, USA
       \AND
       \name Stan Sclaroff \email sclaroff@cs.bu.edu \\
       \addr Department of Computer Science\\
       Boston University\\
       Boston, MA 02215, USA}

\editor{}

\maketitle

\begin{abstract}
We propose a Bayesian approach for recursively estimating the classifier weights in online learning of a classifier ensemble. In contrast with past methods, such as stochastic gradient descent or online boosting, our approach estimates the weights by recursively updating its posterior distribution. For a specified class of loss functions, we show that it is possible to formulate a suitably defined likelihood function and hence use the posterior distribution as an approximation to the global empirical loss minimizer. If the stream of training data is sampled from a stationary process, we can also show that our approach admits a superior rate of convergence to the expected loss minimizer than is possible with standard stochastic gradient descent. In experiments with real-world datasets, our formulation often performs better than state-of-the-art stochastic gradient descent and online boosting algorithms.
\end{abstract}


\section{Introduction}

The basic idea of classifier ensembles is to enhance the performance of individual classifiers by combining them.
In the offline setting, a popular approach to obtain the ensemble weights is to minimize the training error, or a surrogate risk function that approximates the training error. Solving this optimization problem usually calls for various sorts of gradient descent methods. For example, the most successful and popular ensemble technique, boosting, can be viewed in such a way~\citep{FreundSchapire1995,Mason1999,Friedman2001,Telgarsky2012}. Given the success of these ensemble techniques in a variety of batch learning tasks, it is natural to consider extending this idea to the online setting, where the labeled sample pairs $\{\mathbf{x}_t,y_t\}_{t=1}^T$ are presented to and processed by the algorithm sequentially, one at a time.

Indeed, online versions of ensemble methods have been proposed from a spectrum of perspectives. Some of these works focus on close approximation of offline ensemble schemes, such as boosting~\citep{OzaRussell2001,Pelossof2009}. Other methods are based on stochastic gradient descent~\citep{Babenko2009b,Leistner2009,Grbovic2011}. Recently,~\citet{Chen2012} formulated a smoothed boosting algorithm based on the analysis of regret from offline benchmarks. Despite their success in many applications~\citep{Grabner2006, Babenko2009a}, however, there are some common drawbacks of these online ensemble methods, including the lack of a universal framework for theoretical analysis and comparison, and the \emph{ad hoc} tuning of learning parameters such as step size.


In this work, we 
propose 
an online ensemble classification method that is not based on boosting or gradient descent. The main idea is to recursively estimate a posterior distribution of the ensemble weights in a Bayesian manner. We show that, for a given class of loss functions, we can define a likelihood function on the ensemble weights and, with an appropriately formulated prior distribution, we can
generate a posterior mean that closely approximates the empirical loss minimizer. If the stream of training data is sampled from a stationary process, this posterior mean converges to the expected loss minimizer.


Let us briefly explain the rationale for this scheme, which shall be contrasted from the usual Bayesian setup where the likelihood is chosen to describe closely the generating process of the training data. In our framework, we view Bayesian updating as a loss minimization procedure: it provides an approximation to the minimizer of a well-defined risk function. More precisely, this risk minimization interpretation comes from the exploitation of two results in statistical asymptotic theory. First is that, under mild regularity conditions, a Bayesian posterior distribution tends to peak at the maximum likelihood estimate (MLE) of the same likelihood function, as a consequence of the so-called Laplace method~\citep{mackay2003information}. Second, MLE can be viewed as a risk minimizer, where the risk is defined precisely as the expected negative log-likelihood. 
Therefore, given a user-defined loss function, one can choose a suitable log-likelihood as a pure artifact, and apply a corresponding Bayesian update to minimize the risk. We will develop the theoretical foundation that justifies the above rationale.

Our proposed online ensemble classifier learning scheme is straightforward, but powerful in two respects. First, whenever our scheme is applicable, it can approximate the global optimal solution, in contrast with local methods such as stochastic gradient descent (SGD).
Second, assuming the training data is sampled from a stationary process, our proposed scheme possesses a rate of convergence to the expected loss minimizer that is at least as fast as standard SGD. In fact, our rate is faster unless the SGD step size is chosen optimally, which cannot be done \emph{a priori} in the online setting. 
Furthermore, we also found that our method performs better in experiments with finite datasets compared with the averaging schemes in SGD~\citep{PolyakJuditsky1992,Schmidt2013} that have the same optimal theoretical convergence rate as our method. 

In addition to providing a theoretical analysis of our formulation, we also tested our approach on real-world datasets and compared with individual classifiers, a baseline stochastic gradient descent method for learning classifier ensembles, and their averaging variants, as well as state-of-the-art online boosting methods. We found that our scheme consistently achieves superior performance over the baselines and often performs better than state-of-the-art online boosting algorithms, further demonstrating the validity of our theoretical analysis.

\noindent In summary, our contributions are:
\begin{enumerate}
\vspace{-0.1cm}
\item
We propose a Bayesian approach to estimate the classifier weights with closed-form updates for online learning of classifier ensembles.
\vspace{-0.1cm}
\item
We provide theoretical analyses of both the convergence guarantee and the bound on prediction error.
\vspace{-0.1cm}
\item
We compare the asymptotic convergence rate of the proposed framework versus previous gradient descent frameworks thereby demonstrating the advantage of the proposed framework.
\end{enumerate}

This paper is organized as follows. We first briefly discuss the related works. We then state in detail our approach and provide theoretical guarantees in Section~\ref{sec:scheme}. A specific example for solving the online ensemble problem is provided in Section~\ref{sec:examples}, and numerical experiments are reported in Section~\ref{sec:experiments}. We discuss the use of other loss functions for online ensemble learning in Section~\ref{sec:ext} and conclude our paper in Section~\ref{sec:future} with future work. Some technical proofs are left to the Appendix.

\section{Related work}
\label{sec:related}

There is considerable past work on online ensemble learning. Many past works have focused on online learning with concept drift~\citep{WangFanYuHan2003,KolterMaloof2005,KolterMaloof2007,Minku2011}, where dynamic strategies of pruning and rebuilding ensemble members are usually considered. Given the technical difficulty, theoretical analysis for concept drift seems to be underdeveloped. \citet{KolterMaloof2005} proved error bounds for their proposed method, which appears to be the first such theoretical analysis, yet such analysis is not easily generalized to other methods in this category. Other works, such as~\citet{Schapire2001}, and~\citet{Cesa2003}, obtained performance bounds from the perspective of iterative games.

Our work is more closely related to methods that operate in a stationary environment, most notably some online boosting methods.  One of the first methods
was proposed by~\citet{OzaRussell2001}, who showed asymptotic convergence to batch boosting under certain conditions. However, the convergence result only holds for some simple ``lossless" weak classifiers~\citep{Oza2001}, such as Na\"{\i}ve Bayes. Other variants of online boosting have been proposed, such as methods that employ feature selection~\citep{Grabner2006,Liu2007}, semi-supervised learning~\citep{Grabner2008}, multiple instance learning~\citep{Babenko2009a}, and multi-class learning~\citep{Saffari2010}. However, most of these works consider the design and update of weak classifiers beyond that of~\citet{Oza2001} and, thus, do not bear the convergence guarantee therein.
Other methods employ the gradient descent framework, such as Online GradientBoost~\citep{Leistner2009}, Online Stochastic Boosting~\citep{Babenko2009b} and Incremental Boosting~\citep{Grbovic2011}.  There are convergence results given for many of these, which provide a basis for comparison with our framework. In fact, we show that our method compares favorably to gradient descent in terms of asymptotic convergence rate.
%

Other recent online boosting methods (Chen et al., 2012; Beygelzimer et al., 2015) generalize the weak learning assumption to online learning, and can offer theoretical guarantees on the error rate of the learned strong classifier if certain performance assumptions are satisfied for the weak learners.  Our work differs from these approaches, in that our formulation and theoretical analysis focuses on the classes of loss functions, rather than imposing assumptions on the set of weak learners. In particular, we show that the ensemble weights in our algorithm converge asymptotically at an optimal rate to the minimizer of the expected loss.

Our proposed optimization scheme is related to two other lines of work. First is the so-called model-based method for global optimization~\citep{Zlochin2004,RubinsteinKroese2004,Hu2007}. This method iteratively generates an approximately optimal solution as the summary statistic for an evolving probability distribution. It is primarily designed for deterministic optimization, in contrast to the stochastic optimization setting that we consider. Second, our approach is, at least superficially, related to Bayesian model averaging (BMA)~\citep{Hoeting1999}. While BMA is motivated from a model selection viewpoint and aims to combine several candidate models for better description of the data, our approach does not impose any model but instead targets at loss minimization. 

The present work builds on an earlier conference paper~\citep{Bai2014}. We make several generalizations here. First, we remove a restrictive, non-standard requirement on the loss function (which enforces the loss function to satisfy certain integral equality; Assumption 2 in~\citealp{Bai2014}). Second, we conduct experiments that compare our formulation with two variants of the SGD baseline in~\citet{Bai2014}, where the ensemble weights are estimated via two averaging schemes of SGD, 
namely Polyak-Juditsky averaging~\citep{PolyakJuditsky1992} and
Stochastic Averaging Gradient~\citep{Schmidt2013}. Third, we evaluate two additional loss functions for ensemble learning and compare them with the loss function proposed in~\citet{Bai2014}.

\section{Bayesian Recursive Ensemble}
\label{sec:scheme}
We denote the input feature
by $\mathbf{x}$ and its classification label by $y$ ($1$ or $-1$). We assume that we are given $m$ binary weak classifiers $\{c_i(\mathbf{x})\}_{i=1}^m$, and our goal is to find the best ensemble weights $\boldsymbol\lambda=(\lambda_1,\ldots,\lambda_m)$, where $\lambda_i\geq 0$, to construct an ensemble classifier. For now, we do not impose a particular form of ensemble method (we defer this until Section \ref{sec:examples}), although one example form is $\sum_i\lambda_ic_i(\mathbf{x})$.
We focus on online learning, where training data $(\mathbf{x},y)$ comes in sequentially, one at a time at $t=1,2,3, \ldots$.

\subsection{Loss Minimization Formulation}
We formulate the online ensemble learning problem as a stochastic loss minimization problem. We first introduce a loss function at the weak classifier level.
Given a training pair $(\mathbf{x},y)$ and an arbitrary weak classifier $h$, we denote $g:=g(h(\mathbf{x}),y)$ as a non-negative loss function. Popular choices of $g$ include the logistic loss function, hinge loss, ramp loss, zero-one loss, etc. If $h$ is one of the given weak classifiers $c_i$, we will denote $g(c_i(\mathbf{x}),y)$ as $g_i(\mathbf{x},y)$, or simply $g_i$ for ease of notation. Furthermore, we define $g_i^t:=g(c_i^t(\mathbf x^t),y^t)$ where $(\mathbf x^t,y^t)$ is the training sample and $c_i^t$ the updated $i$-th weak classifier at time $t$. To simplify notation, we use $\mathbf g:=(g_1,\ldots,g_m)$ to denote the vector of losses for the weak classifiers, $\mathbf g^t:=(g_1^t,\ldots,g_m^t)$ to denote the losses at time $t$, and $\mathbf g^{1:T}:=(\mathbf g^1,\ldots,\mathbf g^T)$ to denote the losses
up to time $T$.

With the above notation, we let $\ell(\bm\lambda;\mathbf g^t)$ be some ensemble loss function at time $t$, which depends on the ensemble weights and the individual loss of each weak classifier. Then, ideally, the optimal ensemble weight vector $\bm\lambda^*$ should minimize the expected loss $E[\ell(\bm\lambda,\bm g)]$, where the expectation is taken with respect to the underlying distribution of the training data $p(\mathbf{x},y)$. Since this data distribution is unknown, we use the empirical loss as a surrogate:
\begin{equation}
L_T(\boldsymbol\lambda;\mathbf g^{1:T})=\ell_0(\boldsymbol\lambda)+\sum_{t=1}^T\ell(\boldsymbol\lambda;\mathbf g^t)
\label{eq:cumulated}
\end{equation}
where $\ell_0(\bm\lambda)$ can be regarded as an initial loss and can be omitted.

We make a set of assumptions on $L_T$ that are adapted from \citet{Chen1985}: 
\begin{assumption}[Regularity Conditions]
Assume that for each $T$, there exists a $\boldsymbol\lambda_T^*$ that minimizes \eqref{eq:cumulated}, and
\vspace*{-0.5em}
\begin{enumerate}
\item ``local optimality": for each $T$, $\nabla L_T(\bm\lambda_T^*;\mathbf g^{1:T})=0$ and $\nabla^2L_T(\bm\lambda_T^*;\mathbf g^{1:T})$ is positive definite,
\item ``steepness": the minimum eigenvalue of $\nabla^2L_T(\bm\lambda_T^*;\mathbf g^{1:T})$ 
approaches
$\infty$ as $T\to\infty$,
\item ``smoothness": For any $\epsilon>0$, there exists a positive integer $N$ and $\delta>0$ such that for any $T>N$ and $\bm\theta\in H_\delta(\bm\lambda_T^*)=\{\bm\theta:\|\bm\theta-\bm\lambda_T^*\|_2\leq\delta\}$, $\nabla^2L_T(\bm\theta;\mathbf g^{1:T})$ exists and satisfies
{\footnotesize
    $$I-A(\epsilon)\leq\nabla^2L_T(\bm\theta;\mathbf g^{1:T})\left(\nabla^2L_T(\bm\lambda_T^*;\mathbf g^{1:T})\right)^{-1}\leq I+A(\epsilon)$$}
for some positive semidefinite symmetric matrix $A(\epsilon)$ whose largest eigenvalue tends to 0 as $\epsilon\to0$, and the inequalities above are matrix inequalities,
\item ``concentration": for any $\delta>0$, there exists a positive integer $N$ and constants $c,p>0$ such that for any $T>N$ and $\theta\not\in H_\delta(\bm\lambda_T^*)$, we have
    \begin{eqnarray*}
    &&L_T(\bm\theta;\mathbf g^{1:T})-L_T(\bm\lambda_T^*;\mathbf g^{1:T})< \\
    &&c\left((\bm\theta-\bm\lambda_T^*)'\nabla^2L_T(\bm\lambda_T^*;\mathbf g^{1:T})(\bm\theta-\bm\lambda_T^*)\right)^p,
    \end{eqnarray*}
\item ``integrability":
$$\int e^{-L_T(\bm\lambda;\mathbf g^{1:T})}d\bm\lambda<\infty.$$
\end{enumerate} \label{regularity}
\end{assumption}

In the situation where $\ell$ is separable in terms of each component of $\bm\lambda$, i.e.~$\ell(\bm\lambda;\mathbf g)=\sum_{i=1}^mr_i(\lambda_i;\mathbf g)$ and $\ell_0(\bm\lambda)=\sum_{i=1}^ms_i(\lambda_i)$ for some twice differentiable functions $r_i(\cdot;\mathbf g)$ and $s_i(\cdot)$, the assumptions above will depend only on
$f_i(\lambda;\mathbf g^{1:T}):=\sum_{t=1}^Tr_i(\lambda;\mathbf g^t)+s_i(\lambda)$ for each $i$. For example, Condition 3 in Assumption \ref{regularity} reduces to merely checking uniform continuity of each $f_i''(\cdot;\mathbf g^{1:T})$.

Condition 1 in Assumption \ref{regularity} can be interpreted as the standard first and second order conditions for the optimality of $\bm\lambda_T^*$, whereas Condition 3 in essence requires continuity of the Hessian matrix. Conditions 2 and 4 are needed for the use of the Laplace method \citep{mackay2003information}, which, as we will show later, stipulates that the posterior distribution peaks near the optimal solution $\bm\lambda_T^*$ of empirical loss~\eqref{eq:cumulated}.

\subsection{A Bayesian Approach}
\label{sec:framework}

We state our procedure in Algorithm~\ref{alg:framework}. We define $p(\mathbf g|\bm\lambda)=e^{-\ell(\bm\lambda;\mathbf g)}$ and $p(\bm\lambda)=e^{-\ell_0(\bm\lambda)}$.

\begin{algorithm}[htb]
   \caption{Bayesian Ensemble}
   \label{alg:framework}
\begin{algorithmic}
   \STATE {\bfseries Input:} streaming samples $\{(\mathbf{x}^t,y^t)\}_{t=1}^T$
   \STATE \qquad\quad online weak classifiers $\{c_i^t(\mathbf{x})\}_{i=1}^m$
   \STATE \qquad\quad the functions $p(\mathbf{g}|\bm\lambda)$ and $p(\bm\lambda)$
   \STATE {\bfseries Initialize:} hyper-parameters for $p(\mathbf{g}|\bm\lambda)$ and $p(\bm\lambda)$
   \FOR{$t=1$ {\bfseries to} $T$}
   \STATE \begin{itemize}
       \item[] $\forall i$, compute $g_i^t=g(c_i^t(\mathbf{x}^t),y^t)$
       \item[] update for the ``posterior distribution" of $\bm\lambda:$\\
       $p(\bm\lambda|\mathbf{g}^{1:t})\propto p(\mathbf{g}^{t}|\bm\lambda)p(\bm\lambda|\mathbf{g}^{1:t-1}) \propto\prod\limits_{s=1}^tp(\mathbf{g}^{s}|\bm\lambda)p(\bm\lambda)$
       \item[] update the weak classifiers using $(\mathbf{x}^t,y^t)$
       \end{itemize}
   \ENDFOR
\end{algorithmic}
\end{algorithm}
Algorithm~\ref{alg:framework} requires some further explanation:
\begin{enumerate}
\item Our updated estimate for $\bm\lambda$ at each step is the ``posterior mean" for $\bm\lambda$, given by
$$\frac{\int\bm\lambda\prod\limits_{s=1}^tp(\mathbf{g}^{s}|\bm\lambda)p(\bm\lambda)d\bm\lambda}{\int\prod\limits_{s=1}^tp(\mathbf{g}^{s}|\bm\lambda)p(\bm\lambda)d\bm\lambda}$$
\item When the loss function $\ell$ satisfies
\begin{equation}
\int e^{-\ell(\bm\lambda;\mathbf w)}d\mathbf w=1\label{int t}
\end{equation}
and $\ell_0$ satisfies
\begin{equation*}
\int e^{-\ell_0(\mathbf w)}d\mathbf w=1
\end{equation*}
then $p(\mathbf g|\bm\lambda)$ is a valid likelihood function and $p(\bm\lambda)$ a valid prior distribution, so that $p(\bm\lambda|\mathbf g^{1:t})$ as depicted in Algorithm~\ref{alg:framework} is indeed a posterior distribution for $\bm\lambda$ (i.e. the quote-and-quote around ``posterior distribution"  in the algorithm can be removed). In this context, a good choice of $p(\bm\lambda)=e^{-\ell_0(\bm\lambda)}$, e.g. as a conjugate prior for the likelihood $p(\mathbf g|\bm\lambda)=e^{-\ell(\bm\lambda;\mathbf g)}$, can greatly facilitate the computational effort at each step. On the other hand, we also mention that such a likelihood interpretation is not a necessary requirement for Algorithm~\ref{alg:framework} to work, since its convergence analysis relies on the Laplace method, which is non-probabilistic in nature.
\end{enumerate}

Algorithm~\ref{alg:framework} offers the desirable properties characterized by the following theorem. 

\begin{thm}
\label{thm:convergence}
Under Assumption \ref{regularity}, the sequence of random vectors $\bm\lambda_T$ with distributions $p_{\scriptscriptstyle T}(\bm\lambda|\mathbf g^{1:T})$ in Algorithm~\ref{alg:framework} satisfies the asymptotic normality property
\begin{equation}
\left(\nabla^2L_T(\bm\lambda_T^*;\mathbf g^{1:T})\right)^{1/2}(\bm\lambda_T-\bm\lambda_T^*)\stackrel{d}{\to}N(0,1) \label{AN}
\end{equation}
where $\bm\lambda_T$ is interpreted as
 a random variable with distribution $p_{\scriptscriptstyle T}(\bm\lambda|\mathbf g^{1:T})$, and $\stackrel{d}{\to}$ denotes convergence in distribution.  Furthermore, under the uniform integrability condition $\sup_{\scriptscriptstyle T}E_{\bm\lambda_T|\mathbf g^{1:T}}\|\bm\lambda_T-\bm\lambda_T^*\|_1^{1+\epsilon}<\infty$ for some $\epsilon>0$, we have
\begin{equation}
|E_{\bm\lambda_T|\mathbf g^{1:T}}[\bm\lambda_T]-\bm\lambda_T^*|=o\left(\frac{1}{\sigma_T^{1/2}}\right)\label{posterior mean}
\end{equation}
where $E_{\bm\lambda_T|\mathbf g^{1:T}}[\cdot]$ denotes the posterior mean and
$\sigma_T$ is the minimum eigenvalue of the matrix $\nabla^2L_T(\bm\lambda_T^*;\mathbf g^{1:T})$.
\end{thm}

\begin{proof}
Let
$$\tilde L_T(\bm\lambda;\mathbf g^{1:T})=L_T(\bm\lambda;\mathbf g^{1:T})+\log\int e^{-L_T(\bm\lambda;\mathbf g^{1:T})}d\bm\lambda$$
which is well-defined by Condition 5 in Assumption \ref{regularity}. Note that $e^{-\tilde L_T(\bm\lambda;\mathbf g^{1:T})}$ is a valid probability density in $\bm\lambda$ by definition. Moreover, Conditions 1--4 in Assumption \ref{regularity} all hold when $L_T$ is replaced by $\tilde L_T$ (since they all depend only on the gradient of $L_T(\bm\lambda;\mathbf g^{1:T})$ with respect to $\bm\lambda$ or the difference $L_T(\bm\lambda_1;\mathbf g^{1:T})-L_T(\bm\lambda_2;\mathbf g^{1:T})$).

The convergence in \eqref{AN} then follows from Theorem 2.1 in \citet{Chen1985} applied to the sequence of densities $e^{-\tilde L_T(\bm\lambda;\mathbf g^{1:T})}$ for $T=1,2,\dots$. Condition 1 in Assumption \ref{regularity} is equivalent to conditions (P1) and (P2) therein, while Conditions 2 and 3 in Assumption~\ref{regularity} correspond to (C1) and (C2) in~\citet{Chen1985}.
Condition 4 is equivalent to (C3.1), which then implies (C3) there to invoke its Theorem 2.1 to conclude \eqref{AN}.

To show the bound \eqref{posterior mean} we take the expectation on \eqref{AN} to get
\vspace*{-0.25em}
\begin{equation}
\left(\nabla^2L_T(\bm\lambda_T^*;\mathbf g^{1:T})\right)^{\frac{1}{2}}(E_{\bm\lambda_T|\mathbf g^{1:T}}[\bm\lambda_T]-\bm\lambda_T^*)\to0,
\label{interim}
\end{equation}
which is valid because of the uniform integrability condition $\sup_TE_{\bm\lambda_T|\mathbf g^{1:T}}\|\bm\lambda_T-\bm\lambda_T^*\|_1^{1+\epsilon}<\infty$ \citep{durrett10}. Therefore, $E_{\bm\lambda_T|\mathbf g^{1:T}}[\bm\lambda_T]-\bm\lambda_T^*=\left(\nabla^2L(\bm\lambda_T^*;\mathbf g^{1:T})\right)^{-\frac{1}{2}}\mathbf w_T$ where $\mathbf w_T=o(1)$ by \eqref{interim}. But then
\begin{eqnarray*}
&&\left\|\left(\nabla^2L_T(\bm\lambda_T^*;\mathbf g^{1:T})\right)^{-\frac{1}{2}}\mathbf w_T\right\|_1\\
&\leq&\left\|\left(\nabla^2L_T(\bm\lambda_T^*;\mathbf g^{1:T})\right)^{-\frac{1}{2}}\right\|_1\|\mathbf w_T\|_1\\
&\leq&\frac{C}{\sigma_T^{1/2}}\|\mathbf w_T\|_1=o\left(\frac{1}{\sigma_T^{1/2}}\right)
\end{eqnarray*}
where $\|\cdot\|_1$ when applied to matrix is the induced $L_1$-norm. 
This shows \eqref{posterior mean}.
\vspace*{-0.2cm}
\end{proof}

The idea behind \eqref{AN} comes from classical Bayesian asymptotics and is an application of the so-called Laplace method \citep{mackay2003information}. Theorem \ref{thm:convergence} states that given the loss structure satisfying Assumption \ref{regularity}, the posterior distribution of $\bm\lambda$ under our update scheme provides an approximation to the minimizer $\bm\lambda_T^*$ of the cumulative loss at time $T$, as $T$ increases, by tending to a normal distribution peaked at $\bm\lambda_T^*$ with shrinking variance ($\bm\lambda_T^*$ here can be interpreted as the maximum a posterior (MAP) estimate). The bound \eqref{posterior mean} states that this posterior distribution can be summarized using the posterior mean to give a point estimate of $\bm\lambda_T^*$. Moreover, note that $\bm\lambda_T^*$ is the global, not merely local, minimizer of the cumulative loss. This approximation of global optimum highlights a key advantage of the proposed Bayesian scheme over other methods such as stochastic gradient descent (SGD), which only find a local optimum.


The next theorem states another benefit of our Bayesian scheme over standard SGD. Suppose that SGD does indeed converge to the global optimum. Even so, it turns out that our proposed Bayesian scheme converges faster than standard SGD under the assumption of i.i.d. training samples.

\begin{thm}
\label{thm:rate_general}
Suppose Assumption \ref{regularity} holds. Assume also that $\mathbf g^t$ are i.i.d., with $E[\ell(\bm\lambda;\mathbf g)]<\infty$ and $E[\ell(\bm\lambda;\mathbf g)^2]<\infty$. The Bayesian posterior mean produced by Alg.~\ref{alg:framework} converges to $\text{argmin}_{\bm\lambda}E[\ell(\bm\lambda;\mathbf g)]$ strictly faster than standard SGD (supposing it converges to the global minimum), given by
\begin{equation}
\bm\lambda_{T+1}\leftarrow\bm\lambda_T-\epsilon_{\scriptscriptstyle T}K\nabla\ell(\bm\lambda_T;\mathbf g^T)\label{SGD update}
\end{equation}
in terms of the asymptotic variance, except when the step size $\epsilon_{\scriptscriptstyle T}$ and the matrix $K$ is chosen optimally.
\end{thm}

\noindent In Theorem \ref{thm:rate_general}, by asymptotic variance we mean the following: both the sequence of posterior means and the update sequence from SGD possess versions of the central limit theorem, in the form $\sqrt T(\bm\lambda_T-\bm\lambda^*)\stackrel{d}{\to}N(0,\Sigma)$ where $\bm\lambda^*=\text{argmin}_{\bm\lambda}E[\ell(\bm\lambda;\mathbf g)]$. Our comparison is on the asymptotic covariance matrix $\Sigma$ with respect to matrix inequality: for two update schemes with corresponding asymptotic covariance matrices $\Sigma_1$ and $\Sigma_2$, Scheme 1 converges faster than Scheme 2 if $\Sigma_2-\Sigma_1$ is positive definite.

\begin{proof}
The proof follows by combining \eqref{posterior mean} with established central limit theorems for sample average approximation \citep{pasupathy2011stochastic} and stochastic gradient descent (SGD) algorithms. First, let $z(\bm\lambda):=E[\ell(\bm\lambda;\mathbf g)]$, and $\bm\lambda^*:=\text{argmin}_{\bm\lambda}z(\bm\lambda)$. Note that the quantity $\bm\lambda_T^*$ is the minimizer of $\frac{1}{T}\sum_{t=1}^T\ell(\bm\lambda;\mathbf g^t)+\frac{\ell_0(\bm\lambda)}{T}$. Then, together with the fact that $\frac{\ell_0(\bm\lambda)}{T}$ is asymptotically negligible, Theorem 5.9 in \citet{pasupathy2011stochastic} stipulates that $\sqrt T(\bm\lambda_T^*-\bm\lambda^*)\stackrel{d}{\to}N(0,\Sigma)$, where
\begin{equation}
\Sigma=(\nabla^2 z(\bm\lambda))^{-1}Var(\nabla\ell(\bm\lambda;\mathbf g))(\nabla^2 z(\bm\lambda))^{-1}\label{Sigma}
\end{equation}
and $Var(\cdot)$ denotes the covariance matrix.

Now since $\nabla^2L_T(\bm\lambda_T^*;\mathbf g^{1:T})=\sum_{t=1}^T(\nabla^2\ell(\bm\lambda_T^*;\mathbf g^t))$ and $\frac{1}{T}\sum_{t=1}^T(\nabla^2\ell(\bm\lambda_T^*;\mathbf g^t))\to E[\nabla^2\ell(\bm\lambda^*;\mathbf g)]$ by the law of large numbers \citep{durrett10}, we have $\nabla^2L_T(\bm\lambda_T^*;\mathbf g^{1:T})=\Theta(T)$. Then the bound in \eqref{posterior mean} implies that $|E_{\bm\lambda_T|\mathbf g^{1:T}}[\bm\lambda_T]-\bm\lambda_T^*|=o\left(\frac{1}{\sqrt T}\right)$. In other words, the difference between the posterior mean and $\bm\lambda_T^*$ is of smaller scale than $1/\sqrt T$. By Slutsky Theorem \citep{serfling2009approximation}, this implies that $\sqrt T(E_{\bm\lambda_T|\mathbf g^{1:T}}[\bm\lambda_T]-\bm\lambda^*)\stackrel{d}{\to}N(0,\Sigma)$ also.

On the other hand, for SGD \eqref{SGD update}, it is known (e.g.~\citealp{asmussen2007stochastic}) that the optimal step size parameter value is $\epsilon_{\scriptscriptstyle T}=1/T$ and $K=\nabla^2z(\bm\lambda)$, in which case the central limit theorem for the update $\bm\lambda_T$ will be given by $\sqrt T(\bm\lambda_T-\bm\lambda^*)\stackrel{d}{\to}N(0,\Sigma)$ where $\Sigma$ is exactly \eqref{Sigma}. For other choices of step size, either the convergence rate is slower than order $1/\sqrt T$ or the asymptotic variance, denoted by $\tilde\Sigma$, is such that $\tilde\Sigma-\Sigma$ is positive definite. Therefore, by comparing the asymptotic variance, the posterior mean always has a faster convergence unless the step size in SGD is chosen optimally.
\end{proof}

To give some intuition from a statistical viewpoint, Theorem~\ref{thm:rate_general} arises from two layers of approximation of our posterior mean to $\bm\lambda^*$. First, thanks to \eqref{posterior mean}, the difference between posterior mean and the minimizer of cumulative loss $\bm\lambda_T^*$ (which can be interpreted as the MAP) decreases at a rate faster than $1/\sqrt T$. Second, $\bm\lambda_T^*$ converges to $\bm\lambda^*$ at a rate of order $1/\sqrt T$ with the optimal multiplicative constant. This is equivalent to the observation that the MAP, much like the maximum likelihood estimator (MLE), is asymptotically efficient as a statistical estimator.

Putting things in perspective, compared with local methods such as SGD, we have made an apparently stronger set of assumptions (i.e. Assumption \ref{regularity}), which pays off by allowing for stronger theoretical guarantees (Theorems \ref{thm:convergence} and \ref{thm:rate_general}). In the next section we describe an example where a meaningful loss function precisely fits into our framework.

\section{A Specific Example}
\label{sec:examples}
We now discuss in depth a simple and natural choice of loss function and its corresponding likelihood function and prior, which are also used in our experiments in Section \ref{sec:experiments}.
Consider
\begin{equation}
\ell(\bm\lambda;\mathbf g)=\theta\sum_{i=1}^m\lambda_ig_i-\sum_{i=1}^m\log\lambda_i\label{loss}
\end{equation}
The motivation for \eqref{loss} is straightforward: it is the sum of individual losses each weighted by $\lambda_i$. The extra term $\log\lambda_i$ prevents $\lambda_i$ from approaching zero, the trivial minimizer for the first term. The parameter $\theta$ specifies the trade-off between the importance of the first and the second term.
This loss function satisfies Assumption \ref{regularity}. In particular, the Hessian of $L_T$ turns out to not depend on $g^{1:T}$, therefore all conditions of Assumption~\ref{regularity} can be verified easily.

Using the discussion in
Section~\ref{sec:framework}, we choose the exponential likelihood (note that in this definition we add an extra constant term $m\log\theta$ on \eqref{loss}, which does not affect the minimization in any way)
\begin{equation}
p(\mathbf g|\bm\lambda)=\prod_{i=1}^m(\theta\lambda_i)e^{-\theta\lambda_ig_i}\ .
\label{eq:likelihood}
\end{equation}

To facilitate computation, we employ the Gamma prior:
\begin{equation}
p(\bm\lambda)\propto\prod_{i=1}^m\lambda_i^{\alpha-1}e^{-\beta\lambda_i}
\label{eq:prior}
\end{equation}
where $\alpha$ and $\beta$ are the hyper shape and rate parameters. Correspondingly, we pick $\ell_0(\bm\lambda)=\beta\sum_{i=1}^m\lambda_i-(\alpha-1)\sum_{i=1}^m\log\lambda_i$. To be concrete, the cumulative loss in \eqref{eq:cumulated} (disregarding the constant terms) is
$$\beta\sum_{i=1}^m\lambda_i-(\alpha-1)\sum_{i=1}^m\log\lambda_i+\sum_{t=1}^T\left(\theta\sum_{i=1}^m\lambda_ig_i^t-\sum_{i=1}^m\log\lambda_i\right).$$
Now, under conjugacy of \eqref{eq:likelihood} and \eqref{eq:prior}, the posterior distribution of $\bm\lambda$ after $t$ steps is given by the Gamma distribution
$$p(\bm\lambda|\mathbf{g}^{1:t})\propto\prod_{i=1}^m(\lambda_i)^{\alpha+t-1}e^{-(\beta+\theta\sum_{s=1}^tg_i^{s})\lambda_i}\ .$$
Therefore the posterior mean for each $\lambda_i$ is
\begin{equation}
\frac{\alpha+t}{\beta+\theta\sum_{s=1}^tg_i^s}\ .
\label{eq:post_mean}
\end{equation}
We use the following prediction rule at each step:
\begin{equation}
y=
	\left\{
                    \begin{array}{rcl}
			1 \ & \mbox{if}\ \sum\limits_{i=1}^m\lambda_ig_i(\mathbf{x},1)\leq \sum\limits_{i=1}^m\lambda_ig_i(\mathbf{x},-1)\\
			-1 \ & \mbox{otherwise}
		\end{array}
	\right.\\
\label{eq:strong}
\end{equation}
where each $\lambda_i$ is the posterior mean given by \eqref{eq:post_mean}. For this setup, Algorithm~\ref{alg:framework} can be cast as Algorithm~\ref{alg:specification} below, which is to be implemented in Section~\ref{sec:experiments}.
\begin{algorithm}[bht]
   \caption{Closed-form Bayesian Ensemble}
   \label{alg:specification}
\begin{algorithmic}
   \STATE {\bfseries Input:} streaming samples $\{(\mathbf{x}^t,y^t)\}_{t=1}^T$
   \STATE \qquad\quad online weak classifiers $\{c_i^t(\mathbf{x})\}_{i=1}^m$
   \STATE {\bfseries Initialize:} parameters $\theta$ for likelihood~\eqref{eq:likelihood} and parameters $\alpha,\beta$ for prior~\eqref{eq:prior}
   \FOR{$t=1$ {\bfseries to} $T$}
   \STATE \begin{itemize}
       \item[] $\forall i$, compute $g_i^t=g(c_i^t(\mathbf{x}^t),y^t)$, where $g$ is logistic loss function
       \item[] update the posterior mean of $\bm\lambda$ by~\eqref{eq:post_mean}
       \item[] update the weak classifiers according to the particular choice of online weak classifier
       \item[] make prediction by~\eqref{eq:strong} for the next incoming sample
       \end{itemize}
   \ENDFOR
\end{algorithmic}
\end{algorithm}

The following bound provides
further understanding
of the loss function \eqref{loss} and the prediction rule \eqref{eq:strong}, by relating their use with a guarantee on the prediction error:
\begin{thm}
\label{thm:bound}
Suppose that $\mathbf g^t$ are i.i.d., so that $\bm\lambda_T^*$ converges to $\bm\lambda^*:=\text{argmin}_{\bm\lambda}E[\ell(\bm\lambda;\mathbf g)]$ for $\ell$ defined in \eqref{loss}. The prediction error using rule \eqref{eq:strong} with $\bm\lambda^*$ is bounded by
\begin{equation}
P_{(\mathbf{x},y)}(\text{error})\leq m^{\frac{1}{p}}\left(E_{(\mathbf{x},y)}
\left[\left(\sum_{i=1}^m\frac{g_i(\mathbf{x},-y)}{E[g_i(\mathbf{x},y)]}\right)^{\frac{-1}{p-1}}\right]\right)
^{\frac{p-1}{p}}
\label{eq:bound_opt}
\end{equation}
for any $p>1$.
\end{thm}
To make sense of this result, note that the quantity $\frac{1}{E[g_i(\mathbf{x},y)]}g_i(\mathbf{x},-y)$ can be interpreted as a performance indicator of each weak classifier, i.e. the larger it is, the better the weaker classifier is, since a good classifier should have a small loss $E[g_i(\mathbf{x},y)]$ and correspondingly a large $g_i(\mathbf{x},-y)$. As long as there exist some good weak classifiers among the $m$ choices, $\sum_{i=1}^m\frac{g_i(\mathbf{x},-y)}{E[g_i(\mathbf{x},y)]}$ will be large, which leads to a small error bound in \eqref{eq:bound_opt}.
\vspace*{0.2in}
\begin{proof}
Suppose $\bm\lambda$ is used in the strong classifier~\eqref{eq:strong}. Denote $I(\cdot)$ as the indicator function. Consider
\begin{eqnarray*}
&&E_(\mathbf{x},y)\left[\sum_{i=1}^m\lambda_ig_i(\mathbf{x},y)\right]\\
&=&\int\Bigg(\sum_{i=1}^m\lambda_ig_i(\mathbf{x},1)P(y=1|\mathbf{x}){}
+\sum_{i=1}^m\lambda_ig_i(\mathbf{x},-1)P(y=-1|\mathbf{x})\Bigg)dP(\mathbf{x})\\
&\geq&\int\Bigg( I(\sum_{i=1}^m\lambda_ig_i(\mathbf{x},1)>\sum_{i=1}^m\lambda_ig_i(\mathbf{x},-1))
\cdot\sum_{i=1}^m\lambda_ig_i(\mathbf{x},1)P(y=1|\mathbf{x}){}\\
&&{}+I(\sum_{i=1}^m\lambda_ig_i(\mathbf{x},1)<\sum_{i=1}^m\lambda_ig_i(\mathbf{x},-1))
\cdot\sum_{i=1}^m\lambda_ig_i(\mathbf{x},-1)P(y=-1|\mathbf{x})\Bigg)dP(\mathbf{x})\\
&\geq&\int\Bigg( I(\sum_{i=1}^m\lambda_ig_i(\mathbf{x},1)>\sum_{i=1}^m\lambda_ig_i(\mathbf{x},-1))
\cdot\sum_{i=1}^m\lambda_ig_i(\mathbf{x},-1)P(y=1|\mathbf{x}){}\\
&&{}+I(\sum_{i=1}^m\lambda_ig_i(\mathbf{x},1)<\sum_{i=1}^m\lambda_ig_i(\mathbf{x},-1))
\cdot\sum_{i=1}^m\lambda_ig_i(\mathbf{x},1)P(y=-1|\mathbf{x})\Bigg)dP(\mathbf{x})\\
&=&E_{(\mathbf{x},y)}\left[I(\text{error})\sum_{i=1}^m\lambda_ig_i(\mathbf{x},-y)\right]\\
&\geq&P(\text{error})^p
\left(E_{(\mathbf{x},y)}\left[\left(\sum_{i=1}^m\lambda_ig_i(\mathbf{x},-y)\right)^{\frac{-1}{p-1}}\right]\right)^{-(p-1)}
\end{eqnarray*}

the last inequality holds by reverse Holder inequality \citep{Hardy1952}. So
\vspace*{-0.2cm}
\begin{eqnarray*}
P(\text{error})&\leq& \left(E_{(\mathbf{x},y)}\left[\sum_{i=1}^m\lambda_ig_i(\mathbf{x},y)\right]\right)^{\frac{1}{p}}\\
&&\cdot\left(E_{(\mathbf{x},y)}\left[\left(\sum_{i=1}^m\lambda_ig_i(\mathbf{x},-y)\right)^{\frac{-1}{p-1}}\right]\right)
^{\frac{p-1}{p}}
\end{eqnarray*}
and the result~\eqref{eq:bound_opt} follows by plugging in $\lambda_i=\frac{1}{\theta E_{(\mathbf x,y)}[g_i(\mathbf x,y)]}$ for each $i$, the minimizer of $E[\ell(\bm\lambda;\mathbf g)]$, which can be solved directly when $\ell$ is in the form \eqref{loss}.
\vspace*{-0.2cm}
\end{proof}

Finally, in correspondence to Theorem \ref{thm:rate_general}, the standard SGD for \eqref{loss} is written as
\begin{equation}
\lambda_i^{t+1}=\lambda_i^{t}-\frac{\gamma}{t}\left(\theta g_i^t-\frac{1}{\lambda_i^{t}}\right)
\label{eq:SGD}
\end{equation}
where $\gamma$ is a parameter that controls the step size. The following result is a consequence of Theorem~\ref{thm:rate_general} (we give another proof here that reveals more specific details).

\begin{thm}
\label{thm:rate}
Suppose that $\mathbf g^t$ are i.i.d., and $0<E_{(\mathbf x,y)}[g_i(\mathbf x,y)]<\infty$ and $Var_{(\mathbf x,y)}(g_i(\mathbf x,y))<\infty$.
For each $\lambda_i$, the posterior mean given by~\eqref{eq:post_mean} always has a rate of convergence at least as fast as the SGD update \eqref{eq:SGD} in terms of asymptotic variance. In fact, it is strictly better in all situations except when the step size parameter $\gamma$ in \eqref{eq:SGD} is set optimally a priori.
\end{thm}

\begin{proof}
Since for each $i$, $g_i^t$ are i.i.d., the sample mean $(1/T)\sum_{t=1}^T\mathbf g_i^t$ follows a central limit theorem. It can be argued using the delta method \citep{serfling2009approximation} that the posterior mean \eqref{eq:post_mean} satisfies
\begin{eqnarray}
&&\sqrt T\left(\frac{\alpha+T}{\beta+\theta\sum_{t=1}^Tg_i^t}-\frac{1}{\theta E[g_i(\mathbf{x},y)]}\right)\notag\\
&\longrightarrow&N\left(0,\frac{Var(g_i(\mathbf{x},y))}{\theta^2(E[g_i(\mathbf{x},y)])^4}\right)\label{asy Bayesian}
\end{eqnarray}

For the stochastic gradient descent scheme \eqref{eq:SGD}, it would be useful to cast the objective function as $z_i(\lambda_i)=E[\theta\lambda_ig_i-\log\lambda_i]$. Let $\lambda_i^*=\text{argmin}_{\lambda}z_i(\lambda)$ which can be directly solved as $\frac{1}{\theta E[g_i]}$. Then $z_i''(\lambda_i^*)=\frac{1}{{\lambda_i^*}^2}=\theta^2(E[g_i(\mathbf{x},y)])^2$. If the step size $\gamma>\frac{1}{2z''(\lambda_i^*)}$, the update scheme \eqref{eq:SGD} will generate $\lambda_i^T$ that satisfies the following central limit theorem \citep{asmussen2007stochastic,kushner2003stochastic}
\begin{equation}
\sqrt T(\lambda_i^T-\lambda_i^*)\stackrel{d}{\rightarrow}N(0,\sigma_i^2)\label{CLT2}
\end{equation}
where
\begin{equation}
\sigma_i^2=\int_0^\infty e^{(1-2\gamma z_i''(\lambda_i^*))s}\gamma^2Var\left(\theta g_i(\mathbf{x},y)-\frac{1}{\lambda_i^*}\right)ds\label{sigma}
\end{equation}
and $\theta g_i(\mathbf{x},y)-\frac{1}{\lambda_i^*}$ is the unbiased estimate of the gradient at the point $\lambda_i^*$. On the other hand, $\lambda_i^T-\lambda_i^*=\omega_p(\frac{1}{\sqrt T})$ if $\gamma\leq\frac{1}{2z''(\lambda_i^*)}$, i.e.~the convergence is slower than \eqref{CLT2} asymptotically and so we can disregard this case~\citep{asmussen2007stochastic}. Now substitute $\lambda_i^*=\frac{1}{\theta E[g_i]}$ into \eqref{sigma} to obtain
\begin{eqnarray*}
\sigma_i^2&=&\theta^2\gamma^2Var(g_i(\mathbf{x},y))\int_0^\infty e^{(1-2\gamma/\lambda_i^*)s}ds\\
&=&\frac{\theta^2\gamma^2Var(g_i(\mathbf{x},y))}{2\gamma/\lambda_i^*-1}
=\frac{\theta^2\gamma^2Var(g_i(\mathbf{x},y))}{2\gamma\theta^2(E[g_i(\mathbf{x},y)])^2-1}
\end{eqnarray*}
and let $\gamma=\tilde\gamma/\theta^2$, we get
\begin{equation}
\sigma_i^2=\frac{\tilde\gamma^2Var(g_i(\mathbf{x},y))}{\theta^2(2\tilde\gamma(E[g_i(\mathbf{x},y)])^2-1)}\label{eq:asy_SGD}
\end{equation}
if $\tilde\gamma>\frac{\theta^2}{2z''(\lambda_i^*)}=\frac{1}{2(E[g_i(\mathbf{x},y)])^2}$.

We are now ready to compare the asymptotic variance in \eqref{asy Bayesian} and \eqref{eq:asy_SGD}, and show that for all $\tilde\gamma$, the one in \eqref{asy Bayesian} is smaller.
Note that this is equivalent to showing that
$$\frac{Var(g_i(\mathbf{x},y))}{\theta^2(E[g_i(\mathbf{x},y)])^4}
\leq\frac{\tilde\gamma^2Var(g_i(\mathbf{x},y))}{\theta^2(2\tilde\gamma(E[g_i(\mathbf{x},y)])^2-1)}$$
Eliminating the common factors, we have
$$\frac{1}{(E[g_i(\mathbf{x},y)])^2}\leq\frac{\tilde\gamma^2}{2\tilde\gamma-1/(E[g_i(\mathbf{x},y)])^2}$$
and by re-arranging the terms, we have
$$(E[g_i(\mathbf{x},y)])^2\left(\tilde\gamma-\frac{1}{(E[g_i(\mathbf{x},y)])^2}\right)^2\geq0$$
which is always true. Equality holds iff $\tilde\gamma=\frac{1}{(E[g_i(\mathbf{x},y)])^2}$, which corresponds to $\gamma=\frac{1}{\theta^2(E[g_i(\mathbf{x},y)])^2}$.
Therefore, the asymptotic variance in \eqref{asy Bayesian} is always smaller than \eqref{eq:asy_SGD}, unless the step size $\gamma$ is chosen optimally.
\end{proof}


\section{Experiments}
\label{sec:experiments}


We report two sets of binary classification experiments in the online learning setting.
In the first set of experiments, we evaluate our scheme's performance vs.\ five baseline methods: a single baseline classifier, a uniform voting ensemble, and three SGD based online ensemble learning methods.
In the second set of experiments, we compare with three leading online boosting methods: GradientBoost~\citep{Leistner2009}, Smooth-Boost~\citep{Chen2012}, and the online boosting method of~\citet{OzaRussell2001} .


In all experiments, we follow the experimental setup in~\citet{Chen2012}.  Data arrives as a sequence of examples $(\mathbf{x}_1, y_1), . . . ,(\mathbf{x}_T , y_T )$. At each step $t$ the online learner predicts the class label for $\mathbf{x}_t$, then the true label $y_t$ is revealed and used to update the classifier online.
%
%
%
 %
We report the averaged error rate for each evaluated method over five trials of different random orderings of each dataset. The experiments are conducted for two different choices of weak classifiers:  Perceptron and Na\"{\i}ve Bayes. 


In all experiments, we choose the loss function $g$ of our method to be the ramp loss, and set the hyperparameters of our method as $\alpha=\beta=1$ and $\theta=0.1$.
From the expression of the posterior mean~\eqref{eq:post_mean}, the prediction rule~\eqref{eq:strong} is unrelated to the values of $\alpha$, $\beta$ and $\theta$ in the longterm. 
We have observed that the classification performance of our method is not very sensitive with respect to changes in the settings of these parameters.
However, the stochastic gradient descent baseline (SGD)~\eqref{eq:SGD} is sensitive to the setting of $\theta$, and since $\theta=0.1$ works best for SGD we also use $\theta=0.1$ for our method.

\subsection{Comparison with Baseline Methods}
\label{sec:exp_baseline}

In the experimental evaluation, we compare our online ensemble method with five baseline methods:
\begin{enumerate}
\item a single weak classifier ({\scshape Perceptron} or {\scshape Na\"{\i}ve Bayes}),
\item a uniform ensemble of weak classifiers ({\scshape Voting}),
\item an ensemble of weak classifiers where the ensemble weights are estimated via standard stochastic gradient descent ({\scshape SGD}),
\item a variant of (3.) where the ensemble weights are estimated via Polyak averaging~\citep{PolyakJuditsky1992} ({\scshape SGD-avg}), and
\item another variant of (3.) where the ensemble weights are estimated via the Stochastic Average Gradient  method of~\citet{Schmidt2013} ({\scshape SAG}).
\end{enumerate}
%

We use ten binary classification benchmark datasets obtained from the LIBSVM repository\footnote{http://www.csie.ntu.edu.tw/{\textasciitilde}cjlin/libsvmtools/datasets/}.
Each dataset is split into training and testing sets for each random trial, where a training set contains no more than $10\%$ of the total amount of data available for that particular benchmark.   For each experimental trial, the ordering of items in the testing sequence is selected at random, and each online classifier ensemble learning method is presented with the same testing data sequence for that trial.

In each experimental trial, for all ensemble learning methods, we utilize a set of 100 pre-trained weak classifiers that are kept static during the online learning process.
The training set is used in learning these 100 weak classifiers. The same weak classifiers are then shared by all of the ensemble methods, including our method. In order to make weak classifiers divergent, each weak classifier uses a randomly sampled subset of data features as input for both training and testing. The first baseline (single classifier) is learned using all the features.

For all of the benchmarks we observed that the error rate varies with different orderings of the dataset. Therefore, following~\citet{Chen2012}, we report the average error rate over five random trials of different orders of each sequence. In fact, while the error rate may vary according to different orderings of a dataset, it was observed throughout all our experiments that the ranking of performance among different methods is usually consistent.

\begin{table*}[p]
\caption{Experiments of online classifier ensemble using pre-trained Perceptrons as weak classifiers and keeping them fixed online. Mean error rate over five random trials is shown in the table. We compare with five baseline methods: a single Perceptron classifier ({\scshape Perceptron}), a uniform ensemble scheme of weak classifiers ({\scshape Voting}), an ensemble scheme using SGD for estimating the ensemble weights ({\scshape SGD}), an ensemble scheme using the Polyak averaging scheme of SGD~\citep{PolyakJuditsky1992} to estimate the ensemble weights ({\scshape SGD-avg}), and an ensemble scheme using the Stochastic Average Gradient~\citep{Schmidt2013} to estimate the ensemble weights ({\scshape SAG}). Our method attains the top performance for all testing sequences.}
\begin{center}
\begin{small}
\begin{sc}
\scalebox{0.95}{
\begin{tabular}{|l||c|c|c|c|c|c|c|}
\hline
Dataset & \# Examples & Perceptron & Voting & SGD & SGD-avg & SAG & Ours \\
\hline \hline
Heart           & 270   & 0.258 & 0.268 & 0.265 & 0.266 & 0.245 & {\bf 0.239} \\ \hline
Breast-Cancer   & 683   & 0.068 & 0.056 & 0.056 & 0.055 & 0.055 & {\bf 0.050} \\ \hline
Australian      & 693   & 0.204 & 0.193 & 0.186 & 0.187 & 0.171 & {\bf 0.166} \\ \hline
Diabetes        & 768   & 0.389 & 0.373 & 0.371 & 0.372 & 0.364 & {\bf 0.363} \\ \hline
German          & 1000  & 0.388 & 0.324 & 0.321 & 0.323 & 0.315 & {\bf 0.309} \\ \hline
Splice          & 3175  & 0.410 & 0.349 & 0.335 & 0.338 & 0.301 & {\bf 0.299} \\ \hline
Mushrooms       & 8124  & 0.058 & 0.034 & 0.034 & 0.034 & 0.031 & {\bf 0.030} \\ \hline
Ionosphere      & 351   & 0.297 & 0.247 & 0.240 & 0.241 & 0.240 & {\bf 0.236} \\ \hline
Sonar           & 208   & 0.404 & 0.379 & 0.376 & 0.379 & 0.370 & {\bf 0.369} \\ \hline
SVMguide3       & 1284  & 0.382 & 0.301 & 0.299 & 0.299 & 0.292 & {\bf 0.289} \\
\hline
\end{tabular}
}
\label{tab:P_static}
\end{sc}
\end{small}
\end{center}
\end{table*}
\begin{table*}[p]
\caption{Experiments of online classifier ensemble using pre-trained Na\"{\i}ve Bayes as weak classifiers and keeping them fixed online. Mean error rate over five random trials is shown in the table.
We compare with five baseline methods: a single Na\"{\i}ve Bayes classifier ({\scshape Na\"{\i}ve Bayes}), a uniform ensemble scheme of weak classifiers ({\scshape Voting}), an ensemble scheme using SGD for estimating the ensemble weights ({\scshape SGD}), an ensemble scheme using the Polyak averaging scheme of SGD~\citep{PolyakJuditsky1992} to estimate the ensemble weights ({\scshape SGD-avg}), and an ensemble scheme using the Stochastic Average Gradient~\citep{Schmidt2013} to estimate the ensemble weights ({\scshape SAG}).  Our method attains the top performance for all testing sequences.}
\begin{center}
\begin{small}
\begin{sc}
\scalebox{0.95}{
\begin{tabular}{|l||c|c|c|c|c|c|c|}
\hline
dataset & \# Examples & Na\"{\i}ve Bayes & Voting & SGD & SGD-avg & SAG & Ours \\
\hline \hline
Heart           & 270   & 0.232 & 0.207 & 0.214 & 0.215 & 0.206 & {\bf 0.202} \\ \hline
Breast-Cancer   & 683   & 0.065 & 0.049 & 0.050 & 0.049 & 0.048 & {\bf 0.044} \\ \hline
Australian      & 693   & 0.204 & 0.201 & 0.200 & 0.200 & 0.187 & {\bf 0.184} \\ \hline
Diabetes        & 768   & 0.259 & 0.258 & 0.256 & 0.256 & 0.254 & {\bf 0.253} \\ \hline
German          & 1000  & 0.343 & 0.338 & 0.338 & 0.338 & 0.320 & {\bf 0.315} \\ \hline
Splice          & 3175  & 0.155 & 0.156 & 0.155 & 0.155 & {\bf 0.152} & {\bf 0.152} \\ \hline
Mushrooms       & 8124  & 0.037 & 0.066 & 0.064 & 0.064 & 0.046 & {\bf 0.031} \\ \hline
Ionosphere      & 351   & 0.199 & 0.196 & 0.195 & 0.195 & 0.193 & {\bf 0.192} \\ \hline
Sonar           & 208   & 0.338 & 0.337 & 0.337 & 0.337 & 0.337 & {\bf 0.336} \\ \hline
SVMguide3       & 1284  & 0.315 & 0.316 & 0.304 & 0.316 & 0.236 & {\bf 0.215} \\
\hline
\end{tabular}
}
\label{tab:NB_static}
\end{sc}
\end{small}
\end{center}
\end{table*}
\begin{figure}[p]
\centering
    \begin{tabular}{cc}
    \includegraphics[width=0.475\columnwidth, keepaspectratio]{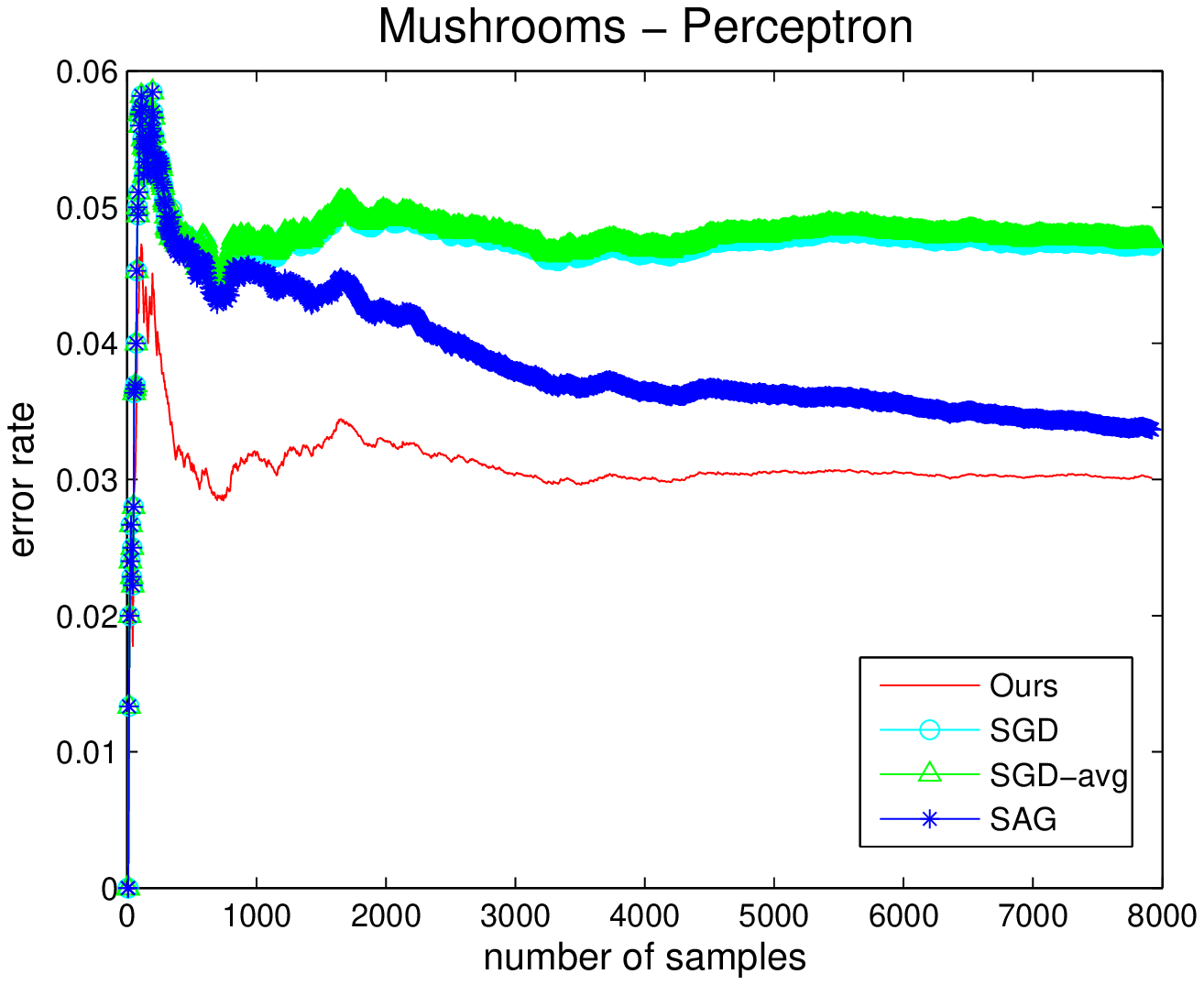} &
    \includegraphics[width=0.475\columnwidth, keepaspectratio]{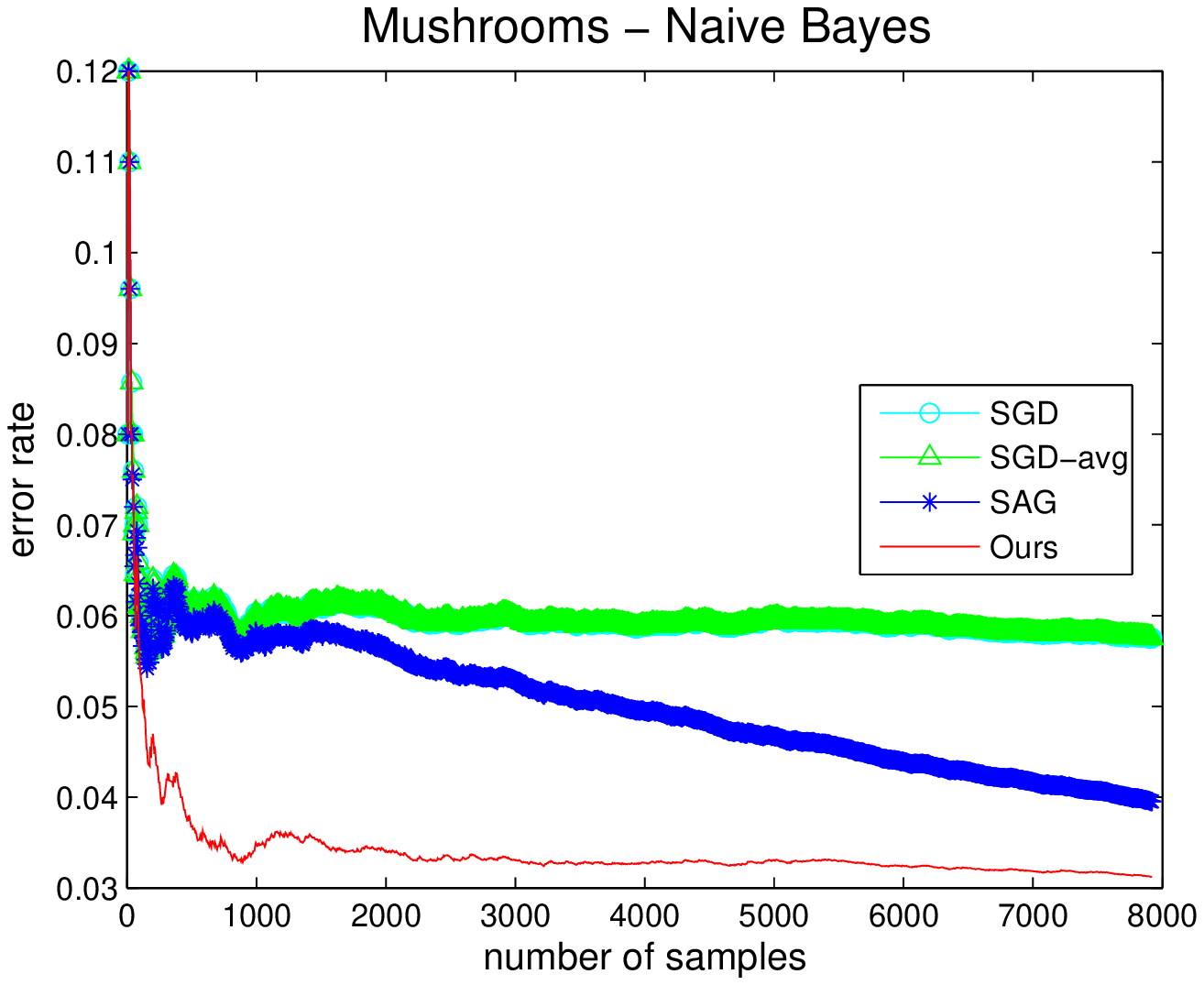}\\
    \includegraphics[width=0.475\columnwidth, keepaspectratio]{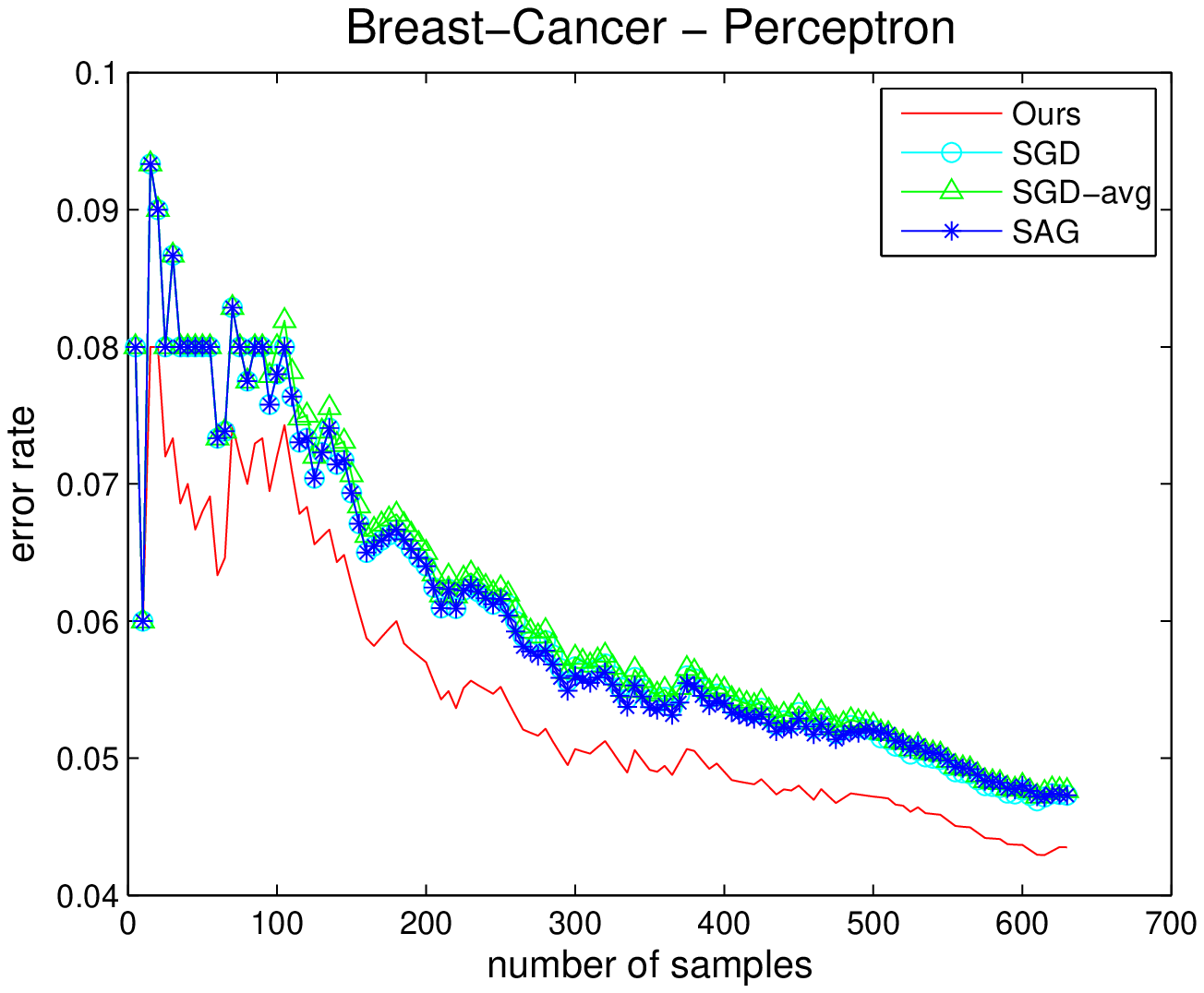} &
    \includegraphics[width=0.475\columnwidth, keepaspectratio]{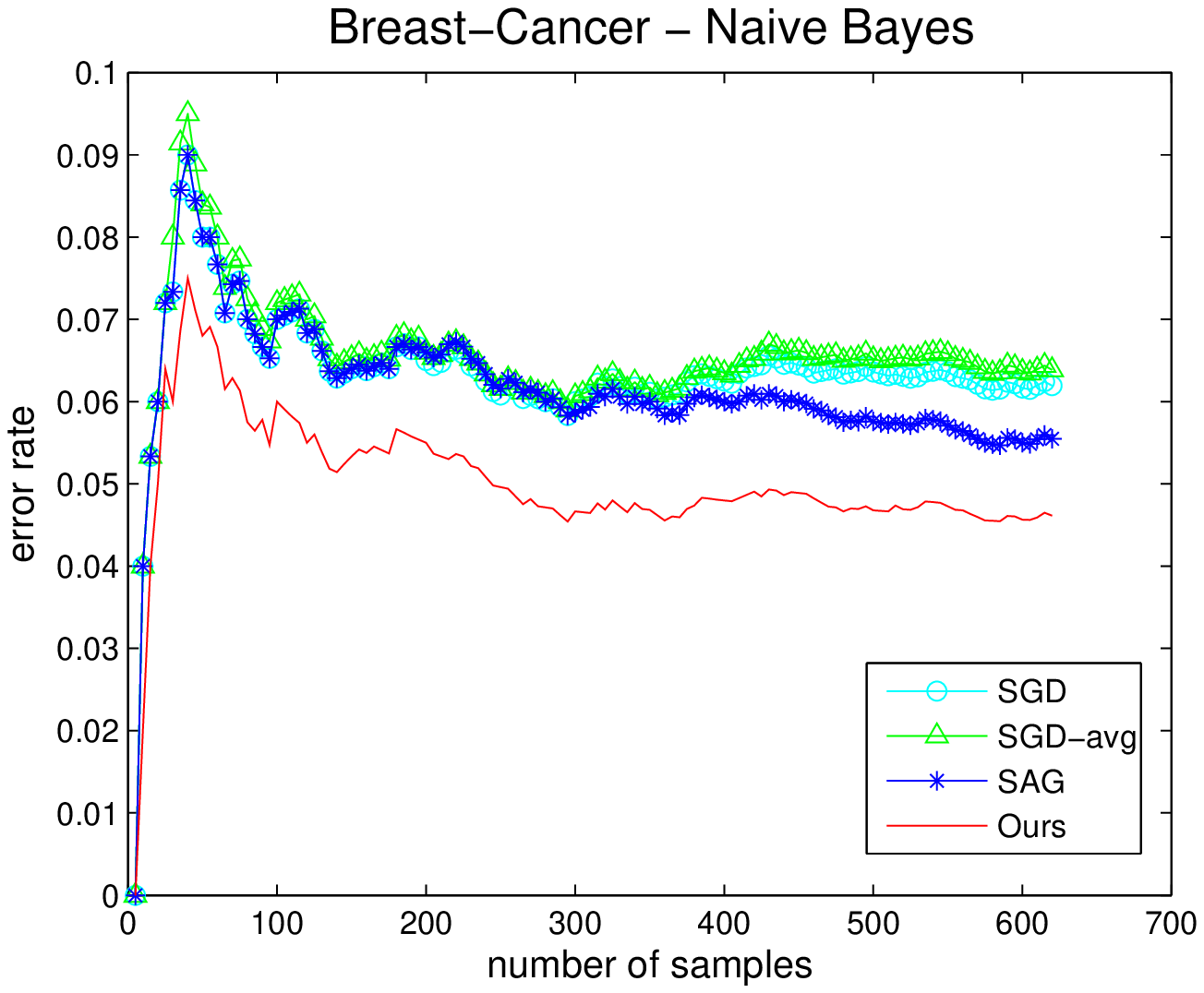}\\
    \includegraphics[width=0.475\columnwidth, keepaspectratio]{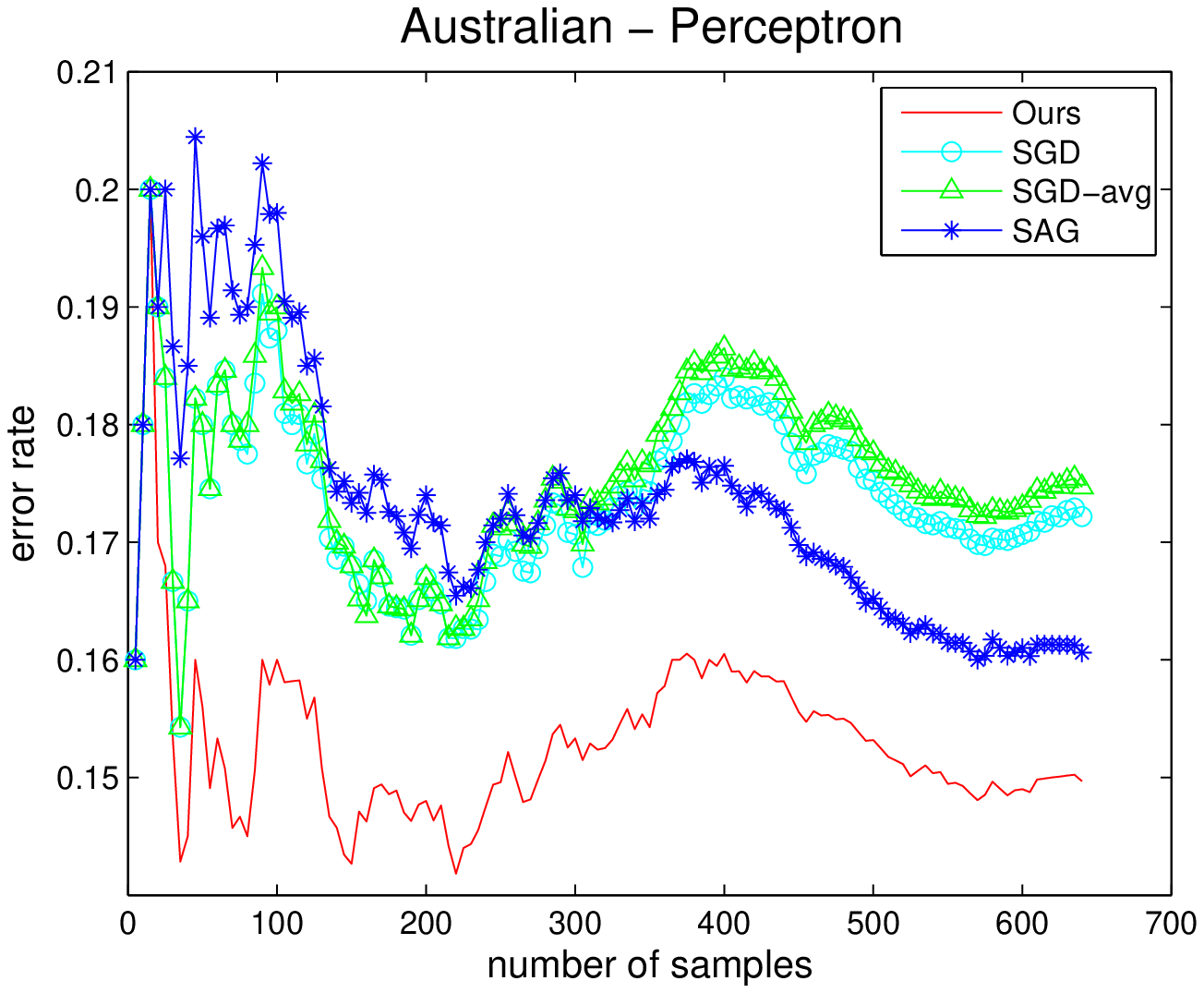} &
    \includegraphics[width=0.475\columnwidth, keepaspectratio]{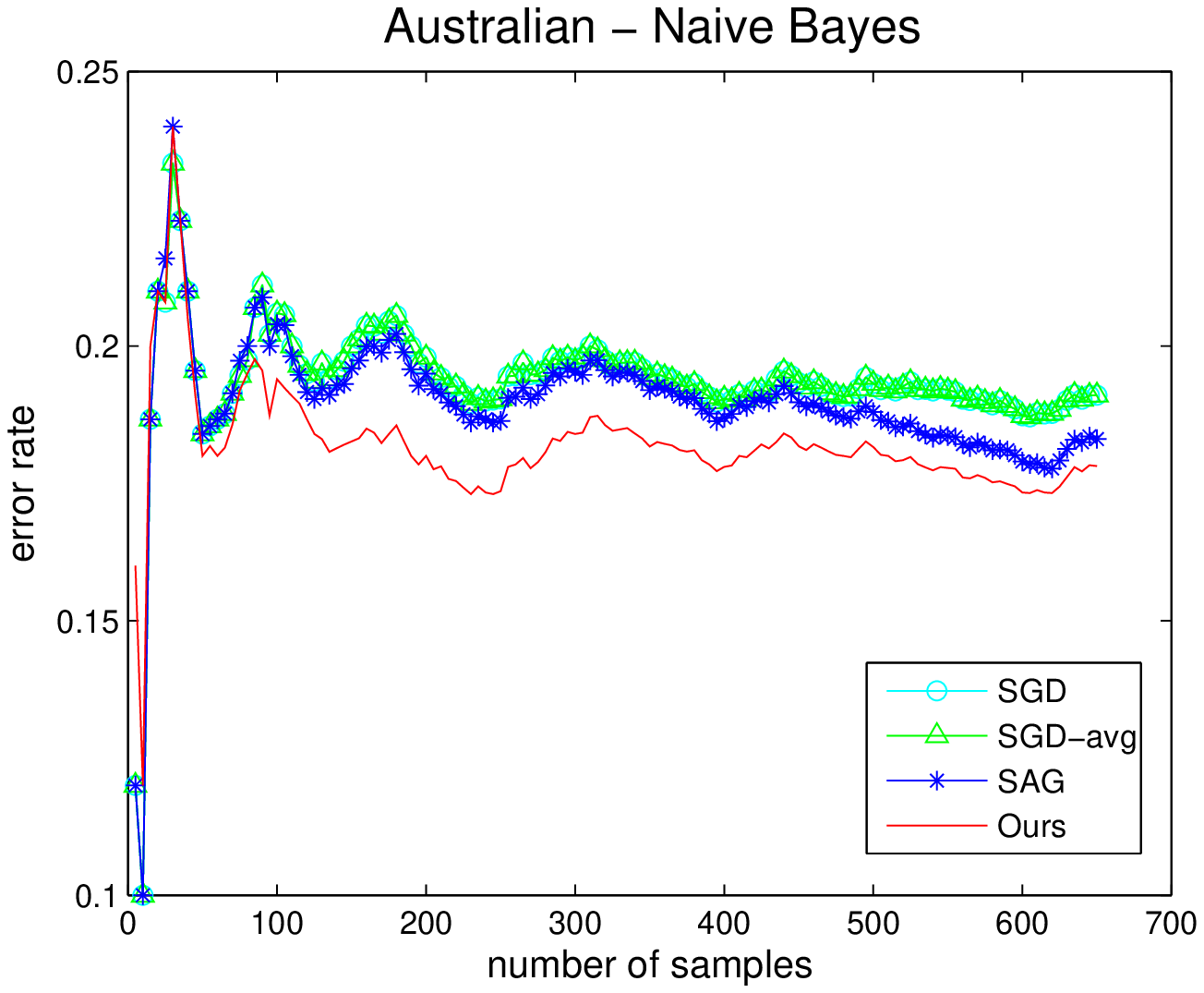}\\
    \end{tabular}
\caption{Plots of the error rate as online learning progresses for three benchmark datasets: Mushrooms, Breast-Cancer, and Australian. (Plots for other benchmarks datasets are provided in the supporting material.) The red curve in each graph shows the error rate for our method, as a function of the number samples processed in the online learning of ensemble weights.  The cyan curves are results from {\scshape SGD} baseline, the green curves are results from the Polyak averaging baseline {\scshape SGD-avg}~\citep{PolyakJuditsky1992}, and the blue curves are results from the Stochastic Average Gradient baseline {\scshape SAG}~\citep{Schmidt2013}.}
\label{fig:error}
\end{figure}

Classification error rates for this experiment are shown in Tables~\ref{tab:P_static} and~\ref{tab:NB_static}.
%
Our proposed method consistently performs the best for all datasets.
Its superior performance against {\scshape Voting} is consistent with the asymptotic convergence analysis in Theorem~\ref{thm:convergence}.
Its superior performance against the {\scshape SGD} baseline is consistent with the convergence rate analysis in Theorem~\ref{thm:rate}.
Polyak averaging ({\scshape SGD-avg}) does not improve the performance of basic SGD in general; this is consistent with the analysis in~\citet{Xu2011} which showed that, despite its optimal asymptotic convergence rate, a huge number of samples may be needed for Polyak averaging to reach its asymptotic region for a randomly chosen step size.
{\scshape SAG}~\citep{Schmidt2013} is a close runner-up to our approach, but it has two limitations:
1) it requires knowing the length of the testing sequence \emph{a priori}, and
2) as noted in~\citet{Schmidt2013}, the step size suggested in the theoretical analysis does not usually give the best result in practice, and thus the authors suggest a larger step size instead. In our experiments, we also found that the improvement of~\citet{Schmidt2013} over the {\scshape SGD} baseline relies on tuning the step size to a value that is greater than that given in the theory. The performance of {\scshape SAG} reported here has taken advantage of these two points.

\begin{figure}[h*]
\centering
    \begin{tabular}{cc}
    \includegraphics[width=0.475\columnwidth, keepaspectratio]{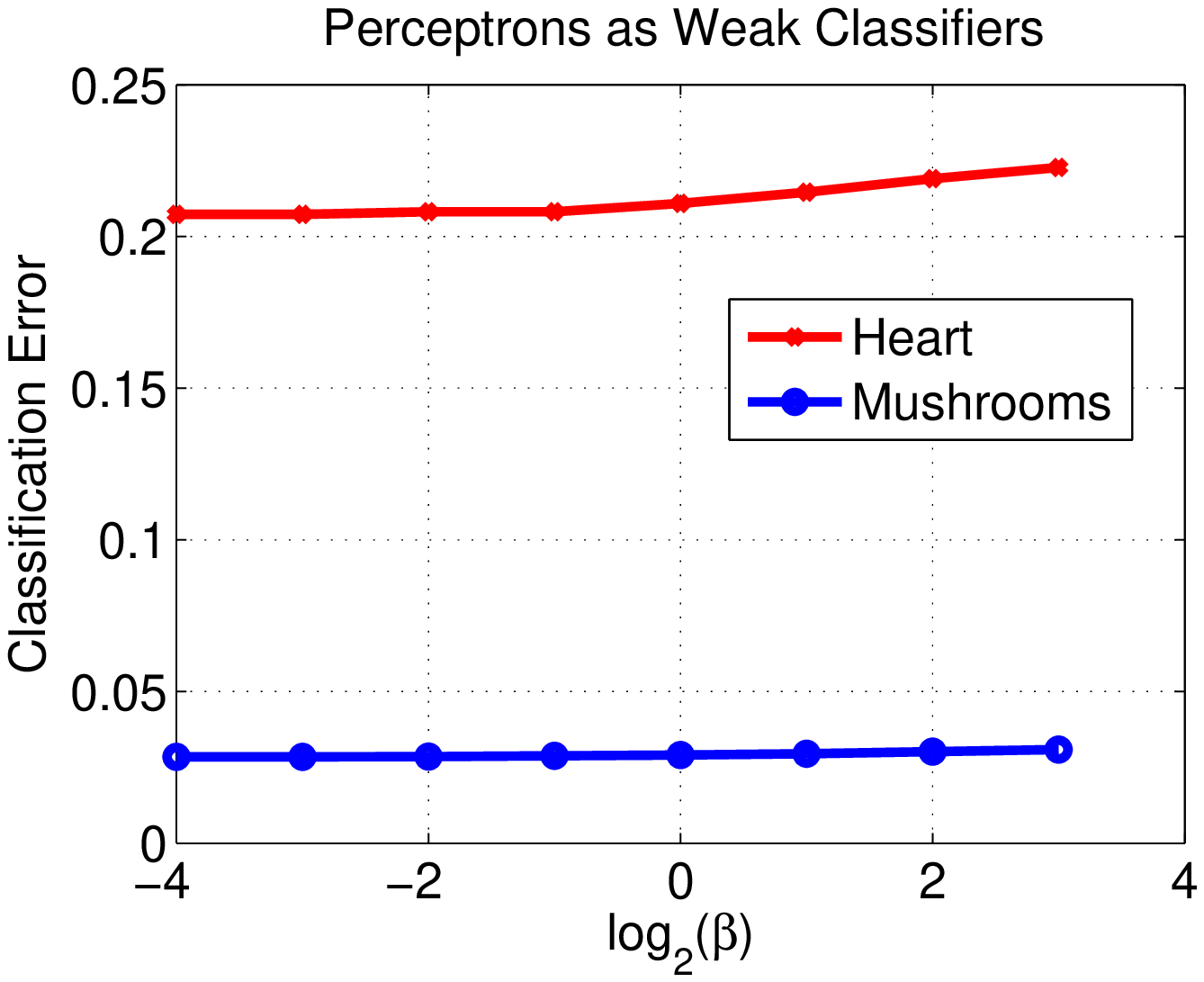} &
    \includegraphics[width=0.475\columnwidth, keepaspectratio]{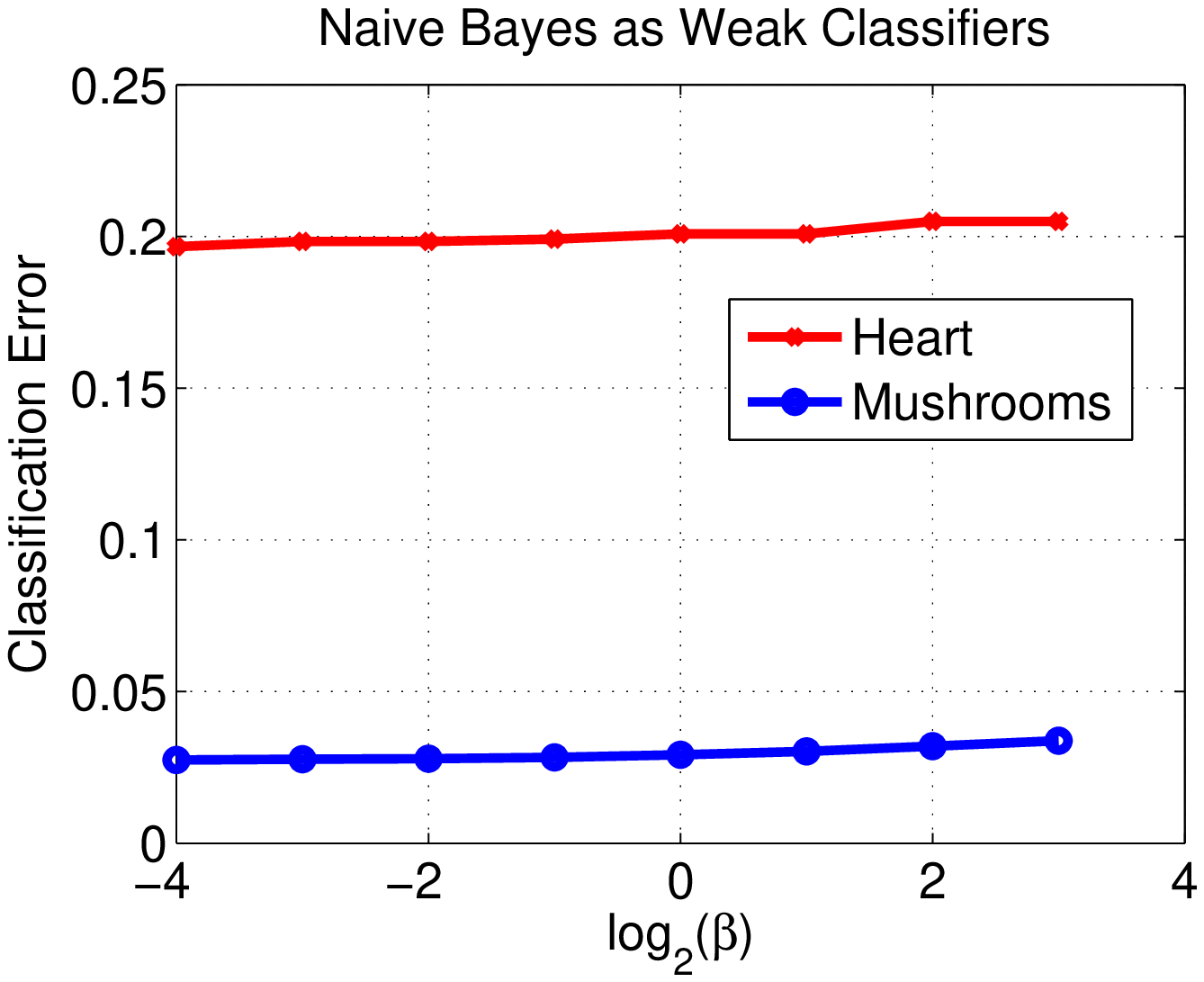}\\
    \end{tabular}
\vspace*{-0.1in}
\caption{Experiments to evaluate different settings of $\beta$ for our online classifier ensemble method,  using pre-trained Perceptrons and Na\"{\i}ve Bayes as weak classifiers. The mean error rate is computed over five random trials for the ``Heart" and ``Mushrooms" datasets. These results are consistent with all other benchmarks tested.}
\label{fig:beta}
\end{figure}
\begin{figure}[h*]
\centering
    \begin{tabular}{cc}
    \includegraphics[width=0.475\columnwidth, keepaspectratio]{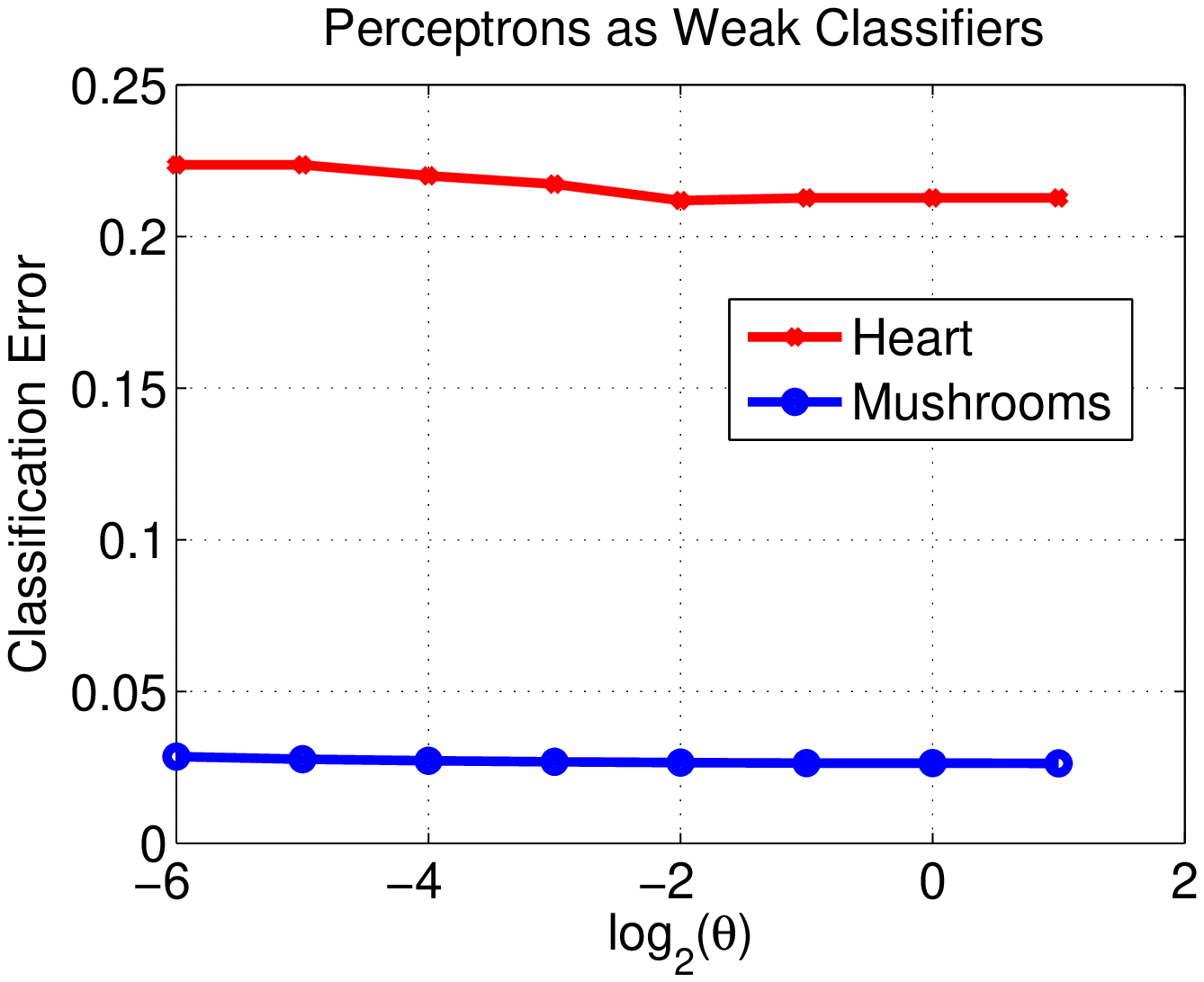} &
    \includegraphics[width=0.475\columnwidth, keepaspectratio]{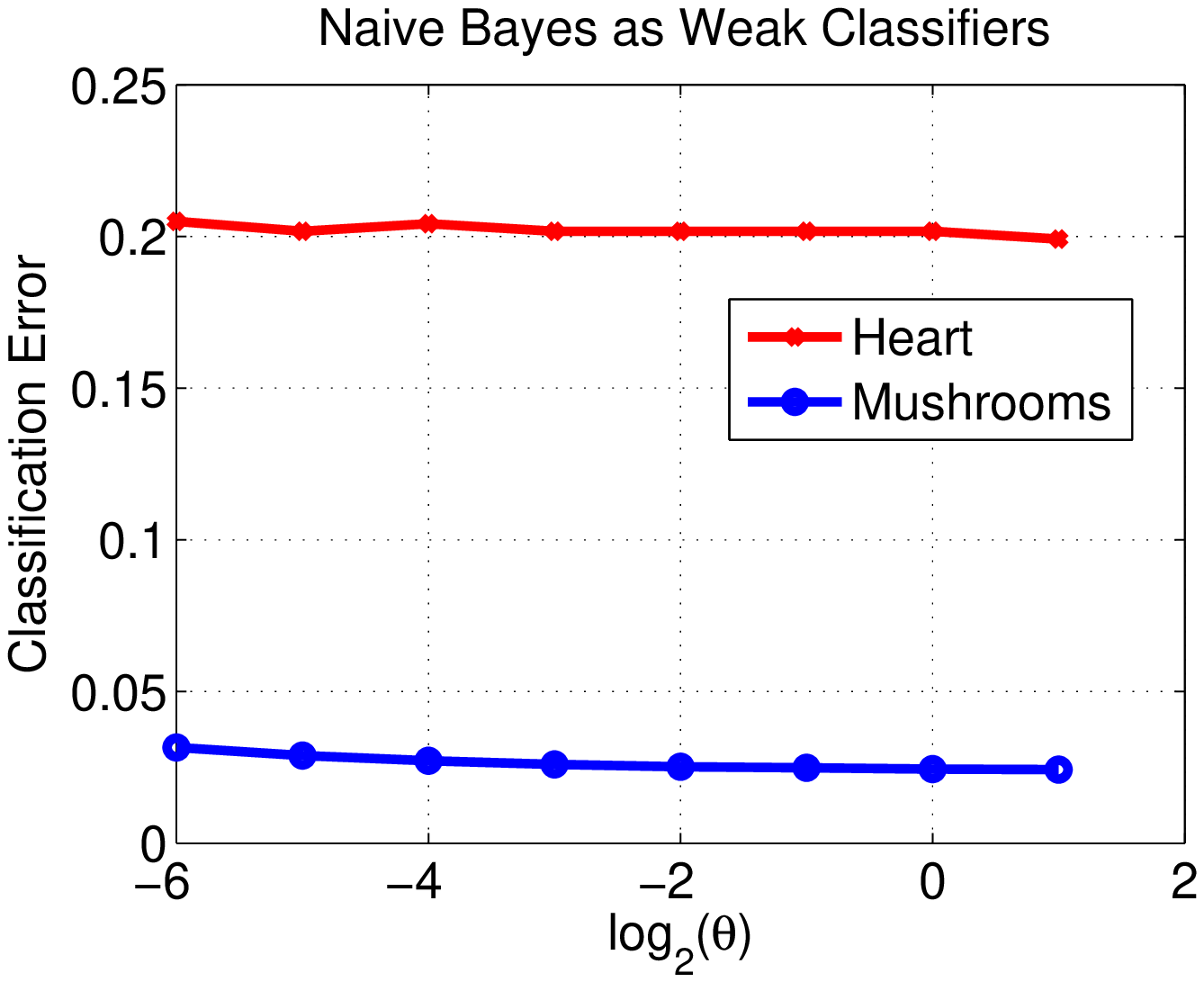}\\
    \end{tabular}
\vspace*{-0.1in}
\caption{Experiments to evaluate different settings of $\theta$ for our online classifier ensemble method,  using pre-trained Perceptrons and Na\"{\i}ve Bayes as weak classifiers. The mean error rate is computed over five random trials for the ``Heart" and ``Mushrooms" datasets. These results are consistent with all other benchmarks tested.}
\label{fig:theta}
\end{figure}
Fig.~\ref{fig:error} shows plots of the convergence of online learning for three of the benchmark datasets.  Plots for the other benchmark datasets are provided in the supplementary material.  Each plot reports the classification error curves of our method, the {\scshape SGD} baseline, Polyak averaging {\scshape SGD-avg}~\citep{PolyakJuditsky1992}, and Stochastic Average Gradient {\scshape SAG}~\citep{Schmidt2013}.   Overall, for all methods, the error rate generally tends to decrease as the online learning process considers more and more samples. As is evident in the graphs, our method tends to attain lowest error rates overall, throughout each training sequence, for the compared methods for these benchmarks.
%
Ideally, as an algorithm converges, the rate of cumulative error should tend to decrease as more samples are processed, approaching the minimal error rate that is achievable for the given set of pre-trained weak classifiers. Yet given the finite size of training sample set, and the randomness caused by different orderings of the sequences, we may not see the ideal monotonic curves. But in general, the trend of curves obtained by our method is consistent with the convergence analysis of Theorem 1.
The online learning algorithm that converges faster should result in curves that go down more quickly in general. Again, given finite samples and different orderings, there is variance, but still, consistent with Theorem 2, the consistently better performance of our formulation vs.\ the compared methods is evident.

Fig~\ref{fig:beta} and Fig.~\ref{fig:theta} show plots for studying the sensitivity of parameter settings of our method. It is clear from the expression of the posterior mean~\eqref{eq:post_mean} that the numerator containing $\alpha$ will be cancelled out in the prediction rule~\eqref{eq:strong}, therefore we just need to study the effect of $\beta$ and $\theta$. We select a short sequence, ``Heart" and a long sequence, ``Mushrooms" as two representative datasets. We plot the classification error rates of our method under different settings of $\beta$ (Fig.~\ref{fig:beta}) and $\theta$ (Fig.~\ref{fig:theta}), averaged over five random trials. It can be observed that the performance of our method is not very sensitive with respect to the changes in the settings of $\beta$ and $\theta$ even for a short sequence like ``Heart" (270 samples). And the performance is more stable to the settings of these parameters for longer sequence like ``Mushrooms" (8124 samples). This observation is consistent with the asymptotic property of our prediction rule~\eqref{eq:strong}.
We observed similar behavior for all the other benchmark datasets we tested.


\subsection{Comparison with Online Boosting Methods}

We further compare our method with a single Perceptron/Na\"{\i}ve Bayes classifier that is updated online, and three representative online boosting methods reported in~\citet{Chen2012}:  {\scshape OzaBoost} is the method proposed by~\citet{OzaRussell2001}, {\scshape OGBoost} is the online GradientBoost method proposed by \citet{Leistner2009}, and {\scshape OSBoost} is the online Smooth-Boost method proposed by~\citet{Chen2012}.  {\scshape Ours-r} is our proposed Bayesian ensemble method for online updated weak classifiers.
All methods are trained and compared following the setup of~\citet{Chen2012}, where for each experimental trial, a set of 100 weak classifiers are initialized and updated online.

We use ten binary classification benchmark datasets that are also used by~\citet{Chen2012}. We discard the ``Ijcnn1" and ``Web Page" datasets from the tables of~\citet{Chen2012}, because they are highly biased with portions of positive samples around $0.09$ and $0.03$ respectively, and even a na\"{\i}ve ``always negative" classifier attains comparably top performance.

\begin{table*}[p]
\caption{Experiments of online classifier ensemble using online Perceptrons as weak classifiers that are updated online. Mean error rate over five trials is shown in the table. We compare with a single online Perceptron classifier ({\scshape Perceptron}) and three representative online boosting methods reported in~\citet{Chen2012}. {\scshape OzaBoost} is the method proposed by~\citet{OzaRussell2001}, {\scshape OGBoost} is the online GradientBoost method proposed by \citet{Leistner2009}, and {\scshape OSBoost} is the online Smooth-Boost method proposed by~\citet{Chen2012}. Our method ({\scshape Ours-R}) attains the top performance for most of the testing sequences.}
\begin{center}
\begin{small}
\begin{sc}
\scalebox{0.95}{
\begin{tabular}{|l|c|c|c|c|c|c|}
\hline
dataset & \# examples & Perceptron & OzaBoost & OGBoost & OSBoost & Ours-r \\
\hline \hline
Heart           & 270   & 0.2489 & 0.2356 & 0.2267 & 0.2356 & {\bf 0.2134} \\ \hline
Breast-Cancer   & 683   & 0.0592 & 0.0501 & 0.0445 & 0.0466 & {\bf 0.0419} \\ \hline
Australian      & 693   & 0.2099 & 0.2012 & 0.1962 & 0.1872 & {\bf 0.1655} \\ \hline
Diabetes        & 768   & 0.3216 & 0.3169 & 0.3313 & 0.3185 & {\bf 0.3098} \\ \hline
German          & 1000  & 0.3256 & 0.3364 & 0.3142 & 0.3148 & {\bf 0.3105} \\ \hline
Splice          & 3175  & 0.2717 & 0.2759 & 0.2625 & 0.2605 & {\bf 0.2584} \\ \hline
Mushrooms       & 8124  & 0.0148 & 0.0080 & 0.0068 & {\bf 0.0060} & 0.0062 \\ \hline
Adult           & 48842 & 0.2093 & 0.2045 & 0.2080 & 0.1994 & {\bf 0.1682} \\ \hline
Cod-RNA         & 488565& 0.2096 & 0.2170 & 0.2241 & 0.2075 & {\bf 0.1934} \\ \hline
Covertype       & 581012& 0.3437 & 0.3449 & 0.3482 & 0.3334 & {\bf 0.3115} \\
\hline
\end{tabular}
}
\label{tab:P_online}
\end{sc}
\end{small}
\end{center}
\end{table*}
\begin{table*}[p]
\caption{Experiments of online classifier ensemble using online Na\"{\i}ve Bayes as weak classifiers that are updated online. Mean error rate over five trials is shown in the table. We compare with a single online Na\"{\i}ve Bayes classifier ({\scshape Na\"{\i}ve Bayes}) and three representative online boosting methods reported in~\citet{Chen2012}. {\scshape OzaBoost} is the method proposed by~\citet{OzaRussell2001}, {\scshape OGBoost} is the online GradientBoost method proposed by \citet{Leistner2009}, and {\scshape OSBoost} is the online Smooth-Boost method proposed by~\citet{Chen2012}. Our method ({\scshape Ours-R}) attains the top performance for 7 out of 10 testing sequences. For ``Cod-RNA" our implementation of the Na\"{\i}ve Bayes baseline was unable to duplicate the reported result; ours gave 0.2555 instead.}
%
\begin{center}
\begin{small}
\begin{sc}
\scalebox{0.95}{
\begin{tabular}{|l|c|c|c|c|c|c|}
\hline
dataset & \# examples & Naive Bayes & OzaBoost & OGBoost & OSBoost & Ours-r \\
\hline \hline
Heart           & 270   & 0.1904 & 0.2570 & 0.3037 & 0.2059 & {\bf 0.1755} \\ \hline
Breast-Cancer   & 683   & 0.0474 & 0.0635 & 0.1004 & 0.0489 & {\bf 0.0408} \\ \hline
Australian      & 693   & 0.1751 & 0.2133 & 0.2826 & 0.1849 & {\bf 0.1611} \\ \hline
Diabetes        & 768   & 0.2664 & 0.3091 & 0.3292 & 0.2622 & {\bf 0.2467} \\ \hline
German          & 1000  & 0.2988 & 0.3206 & 0.3598 & 0.2730 & {\bf 0.2667} \\ \hline
Splice          & 3175  & 0.2520 & 0.1563 & 0.1863 & 0.1370 & {\bf 0.1344} \\ \hline
Mushrooms       & 8124  & 0.0076 & 0.0049 & 0.0229 & {\bf 0.0029} & 0.0054 \\ \hline
Adult           & 48842 & 0.2001 & 0.1912 & 0.1878 & {\bf 0.1581} & 0.1658 \\ \hline
Cod-RNA         & 488565& 0.2206$*$ & 0.0796 & 0.0568 & {\bf 0.0581} & 0.2552 \\ \hline
Covertype       & 581012& 0.3518 & 0.3293 & 0.3732 & 0.3634 & {\bf 0.3269} \\
\hline
\end{tabular}
}
\label{tab:NB_online}
\end{sc}\\
\end{small}
\end{center}
\end{table*}

The error rates for this experiment are shown in Tables~\ref{tab:P_online} and~\ref{tab:NB_online}.  As can be seen, our method outperforms competing methods using the Perceptron weak classifier in nearly all the benchmarks tested.  Moreover, our method performs among the best for the Na\"{\i}ve Bayes weak classifier.  It is worth noting that our method is the only one that outperforms the single classifier baseline in all benchmark datasets, which further confirms the effectiveness of the proposed ensemble scheme.

We also note that despite our best efforts to align both the weak classifier construction and experimental setup with competing methods~\citep{Chen2012,Chen2013}, there are inevitably differences in weak classifier construction. Firstly, given that our method only focuses on optimizing the ensemble weights, each incoming sample is treated equally in the update of all weak classifiers, while all three online boosting methods adopt more sophisticated weighted update schemes for the weak classifiers, where the sample weight is dynamically adjusted during each round of update. Secondly, in order to make weak classifiers different from each other, our weak classifiers use only a subset of input features, while weak classifiers of competing methods use all features and are updated differently.  As a result, the weak classifiers used by our method are actually weaker than in competing methods. Nevertheless, our method often compares favorably.


\section{Additional Loss Functions for Online Ensemble Learning}\label{sec:ext}
We discuss other loss functions that fit into our Bayesian online ensemble learning framework. Note that the loss function~\eqref{loss} given in Section~\ref{sec:examples} is very simple, to the extent that the surrogate empirical loss~\eqref{eq:cumulated} at each step can be directly minimized in closed-form. 
To demonstrate the flexibility of our framework, the empirical losses in the two examples we give below cannot be minimized directly, but they are still effectively solvable using our approach.
\begin{enumerate}
\item Consider the loss function
\begin{eqnarray}
\ell(\bm\lambda;\mathbf g)&=&\sum\limits_{i=1}^m(1-\lambda_i)\log g_i+\theta\sum\limits_{i=1}^m g_i\nonumber\\
&+&\sum\limits_{i=1}^m\log\Gamma(\lambda_i)-(\log\theta)\sum\limits_{i=1}^m\lambda_i
\label{eq:loss1}
\end{eqnarray}
where $\theta>0$ is a fixed parameter. The corresponding likelihood is given by the following product of Gamma distributions
\begin{equation}
p(\mathbf g|\bm\lambda)=\prod_{i=1}^m\frac{\theta^{\lambda_i}}{\Gamma(\lambda_i)}g_i^{\lambda_i-1}e^{-\theta g_i}
\label{eq:likelihood_loss1}
\end{equation}
A conjugate prior for $\bm\lambda$ is available, in the form
$$p(\bm\lambda)\propto\prod_{i=1}^m\frac{a^{\lambda_i-1}\theta^{c\lambda_i}}{\Gamma(\lambda_i)^b}$$
where $a,b,c>0$ are hyperparameters. The posterior distribution of $\bm\lambda$ after $t$ steps is given by the Gamma distribution
\begin{equation}
p(\bm\lambda|\mathbf{g}^{1:t})\propto
\prod_{i=1}^m\frac{(a\prod\limits_{s=1}^tg_i^{s})^{\lambda_i-1}\theta^{(c+t)\lambda_i}}{\Gamma(\lambda_i)^{(b+t)}}
\label{eq:post_loss1}
\end{equation}
Note that given posterior~\eqref{eq:post_loss1}, the posterior mean for each $\lambda_i$ is not available in closed-form, but it can be computed using standard numerical integration procedures, such as those provided in the Matlab Mathematics Toolbox (it only involves one-dimensional procedures because of the independence among the $\bm\lambda$). The corresponding prediction rule at each step is given by
\begin{equation*}
y=
	\left\{
                    \begin{array}{rcl}
			1 \ & \text{if}\ \sum\limits_{i=1}^m(1-\lambda_i)\log \frac{g_i(\mathbf{x},1)}{g_i(\mathbf{x},-1)} + \theta\sum\limits_{i=1}^m (g_i(\mathbf{x},1)-g_i(\mathbf{x},-1))\leq 0\\
			-1 \ & \text{otherwise}
		\end{array}
	\right.\\
\label{eq:strong_loss1}
\end{equation*}
Note that the likelihood function~\eqref{eq:likelihood_loss1} of $g$ is a Gamma distribution, which has support $(0,\infty)$. For computational convenience, instead of choosing the ramp loss for $g$ as in Section~\ref{sec:examples}, we can choose $g$ to be the logistic function.

\item
We can extend the ensemble weights to include two correlated parameters for each weight, i.e., $\lambda_i=(\alpha_i,\beta_i)$. In this case, we may define the loss function as
\begin{eqnarray}
\ell(\bm\alpha,\bm\beta;\mathbf g)&=&\sum\limits_{i=1}^m\beta_i g_i+\sum\limits_{i=1}^m(1-\alpha_i)\log g_i\nonumber\\
&+&\sum\limits_{i=1}^m\log\Gamma(\alpha_i)-\sum\limits_{i=1}^m\alpha_i\log\beta_i
\label{eq:loss2}
\end{eqnarray}
with the corresponding Gamma likelihood
\begin{equation}
p(\mathbf g|\bm\alpha,\bm\beta)=\prod_{i=1}^m\frac{\beta_i^{\alpha_i}}{\Gamma(\alpha_i)}g^{\alpha_i-1}e^{-\beta_i g_i}
\label{eq:likelihood_loss2}
\end{equation}
A conjugate prior is available for $\bm\alpha$ and $\bm\beta$ jointly
$$p(\bm\alpha,\bm\beta)\propto\prod_{i=1}^m\frac{p^{\alpha_i-1}e^{-q\beta_i }}{\Gamma(\alpha_i)^r\beta_i^{-\alpha_i s}}$$
where $p,q,r,s$ are hyperparameters. The posterior distribution of $\bm\alpha$ and $\bm\beta$ after $t$ steps is given by the Gamma distribution
\begin{equation}
p(\bm\alpha,\bm\beta|\mathbf g^{1:t})\propto\prod_{i=1}^m\frac{(p\prod\limits_{s=1}^tg_i^{s})^{\alpha_i-1}e^ {-(q+\sum_{s=1}^tg_i^{s})\beta_i}}{\Gamma(\alpha_i)^{(r+t)}\beta_i^{-\alpha_i(s+t)}}
\label{eq:post_loss2}
\end{equation}
Again, the posterior mean for~\eqref{eq:post_loss2} is not available in closed-form and we can approximate it using numerical methods. The corresponding prediction rule at each step is given by
\begin{equation*}
y=
	\left\{
                    \begin{array}{rcl}
			1 \ & \mbox{if}\ \sum\limits_{i=1}^m(1-\alpha_i)\log \frac{g_i(\mathbf{x},1)}{g_i(\mathbf{x},-1)} + \sum\limits_{i=1}^m\beta_i (g_i(\mathbf{x},1)-g_i(\mathbf{x},-1)) \leq 0\\
			-1 \ & \mbox{otherwise}
		\end{array}
	\right.\\
\label{eq:strong_loss2}
\end{equation*}
\end{enumerate}

Note that both of these two loss functions satisfy Assumption \ref{regularity}. Similar as the example proposed in Section~\ref{sec:examples}, the Hessian of $L_T$ turns out to not depend on $g^{1:T}$, therefore all conditions of Assumption~\ref{regularity} can be verified easily. 
As a result, applying Algorithm~\ref{alg:framework} on these two loss functions for solving the online ensemble learning problem also possesses the convergence properties given by Theorems~\ref{thm:convergence} and~\ref{thm:rate_general}.

We follow the experimental setup of Section~\ref{sec:exp_baseline} to compare our proposed loss~\eqref{loss} with the additional losses~\eqref{eq:loss1} and~\eqref{eq:loss2} discussed here, using pre-trained Perceptron and Na\"{i}ve Bayes as weak classifiers.
The loss function $g$ for weak classifier $c$ is chosen as a logistic function of $y\cdot c(x)$. According to the posterior update rules given in~\eqref{eq:post_loss1} and~\eqref{eq:post_loss2}, hyper parameters $b,c$ and $r,s$ will keep increasing as online learning proceeds. However, we observe in practice that the numerical integration
of posterior means based on posterior distributions~\eqref{eq:post_loss1} and~\eqref{eq:post_loss2} will not converge if the values of hyper parameters $b,c,r,s$ are too large. In our experiments, we set upper bounds for these parameters. In particular, we set the upper bound for $b$ and $c$ as $1000$, the upper bound for $r$ and $s$ as $200.5$ and $200$ respectively (Since $s$ should be strictly less than $r$, we use the following initialization: $s=1$, $r=1.5$, as suggested by~\citealp{Fink1997}).

Averaged classification error rate over five trials for this experiment is shown in Table~\ref{tab:loss1}. Note that the result in this table should not be directly compared with those reported in Tables~\ref{tab:P_static} and~\ref{tab:NB_static}, given the loss function $g$ for weak classifiers is chosen differently.
We observe that loss~\eqref{eq:loss2} works slightly better than loss~\eqref{eq:loss1}, which is reasonable given more parameters in the formula of~\eqref{eq:loss2}. This advantage also leads to a superior performance to loss~\eqref{loss} proposed in Section 4 for shorter sequences, such as ``Heart", ``Ionosphere" and ``Sonar". However, for longer sequences, loss~\eqref{loss} still has some advantage because of the closed-form posterior mean.


\begin{table*}[htb]
\caption{Experiments of online classifier ensemble using pre-trained Perceptrons/Na\"{\i}ve Bayes as weak classifiers and keeping them fixed online. Mean error rate over five random trials is shown in the table.
We compare our method using the proposed loss function~\eqref{loss} with alternative losses defined by~\eqref{eq:loss1} and~\eqref{eq:loss2}. In general, the loss function~\eqref{loss} that enables closed-form posterior mean performs the best.}
\vskip -0.1in
\begin{center}
\begin{small}
\begin{sc}
\scalebox{0.95}{
\begin{tabular}{|l||c|c|c|c||c|c|c|}
        \hline
         & & \multicolumn{3}{c||}{Perceptron weak learner} & \multicolumn{3}{c|}{Na\"{i}ve Bayes weak learner} \\
        \cline{1-8}
        dataset &  \# examples & loss~\eqref{loss} & loss~\eqref{eq:loss1} & loss~\eqref{eq:loss2} & loss~\eqref{loss} & loss~\eqref{eq:loss1} & loss~\eqref{eq:loss2} \\
        \hline \hline
        Heart           & 270   & 0.203 & 0.208 & {\bf 0.198} & 0.197 & 0.204 & {\bf 0.196} \\ \hline
        Breast-Cancer   & 683   & {\bf 0.065} & 0.070 & 0.068 & {\bf 0.045} & 0.050 & 0.046 \\ \hline
        Australian      & 693   & {\bf 0.183} & 0.207 & 0.200 & {\bf 0.191} & 0.209 & 0.203 \\ \hline
        Diabetes        & 768   & 0.301 & 0.307 & {\bf 0.300} & 0.285 & 0.287 & {\bf 0.284} \\ \hline
        German          & 1000  & {\bf 0.338} & 0.347 & 0.348 & {\bf 0.292} & {\bf 0.292} & 0.293 \\ \hline
        Splice          & 3175  & {\bf 0.390} & 0.418 & 0.418 & {\bf 0.144} & 0.150 & 0.150 \\ \hline
        Mushrooms       & 8124  & {\bf 0.028} & 0.032 & 0.031 & {\bf 0.025} & 0.047 & 0.046 \\ \hline
        Ionosphere      & 351   & 0.293 & 0.295 & {\bf 0.259} & {\bf 0.171} & 0.172 & {\bf 0.171} \\ \hline
        Sonar           & 208   & 0.385 & 0.391 & {\bf 0.380} & {\bf 0.301} & 0.302 & 0.303 \\ \hline
        SVMguide3       & 1284  & {\bf 0.265} & 0.278 & 0.276 & {\bf 0.222} & 0.226 & 0.225\\ \hline
        \end{tabular}
        }
\label{tab:loss1}
\end{sc}
\end{small}
\end{center}
\vskip -0.1in
\end{table*}

\section{Conclusion}
\label{sec:future}
We proposed a Bayesian approach for online estimation of the weights of a classifier ensemble. This approach was based on an empirical risk minimization property of the posterior distribution, and involved suitably choosing the likelihood function based on a user-defined choice of loss function. We developed the theoretical foundation, and identified the class of loss functions, for which the update sequence generated by our approach converged to the stationary risk minimizer. We demonstrated that, unlike standard SGD, the convergence guarantee was global and that the rate was optimal in a well-defined asymptotic sense. Moreover, experiments on real-world datasets demonstrated that our approach compared favorably to state-of-the-art SGD methods and online boosting methods.
In future work, we will study further generalization of the scope of loss functions, and the extension of our framework to non-stationary environments.

\bibliography{jmlr}

\end{document}